\documentclass{article}
\csname remappendix\endcsname

\usepackage{arxiv}

\usepackage[utf8]{inputenc} %
\usepackage[T1]{fontenc}    %
\usepackage{hyperref}       %
\usepackage{url}            %
\usepackage{booktabs}       %
\usepackage{amsfonts}       %
\usepackage{nicefrac}       %
\usepackage{microtype}      %
\usepackage[dvipsnames]{xcolor}         %
\usepackage{longtable}
\usepackage{wrapfig}
\usepackage{array}
\usepackage{enumerate}
\usepackage{graphicx}
\usepackage{multicol}
\usepackage{enumitem}
\usepackage{tabularray}
\usepackage{hhline}
\usepackage{stackengine}
\usepackage{mwe}
\usepackage{multirow}
\usepackage{ragged2e}

\usepackage{amsmath}
\usepackage{amsthm}
\newtheorem{theorem}{Theorem}
\newtheorem{proposition}{Proposition}
\newtheorem{lemma}[theorem]{Lemma}
\theoremstyle{definition}
\newtheorem{definition}{Definition}
\newtheorem{example}{Example}

\newcommand{\dd}{\mathrm{d}}
\usepackage{tikz}
\usetikzlibrary{bayesnet}
\usepackage{cleveref}
\usepackage{listings}
\usepackage{multirow}
\usepackage{siunitx}
\usepackage{etoolbox}
\robustify\bfseries
\usepackage[ruled]{algorithm2e}
\usepackage{float}
\usepackage{scalerel}

\DeclareMathOperator*{\amper}{\scalebox{1.5}{\&}}

\newcolumntype{P}[1]{>{\centering\arraybackslash}p{#1}}
\newcolumntype{M}[1]{>{\centering\arraybackslash}m{#1}}

\author{%
  Feynman Liang\\
  Department of Statistics\\
  University of California, Berkeley\\
  \texttt{feynman.liang@gmail.com} \\
  \And
  Liam Hodgkinson \\
  School of Mathematics and Statistics \\
  University of Melbourne, Australia \\
  \texttt{lhodgkinson@unimelb.edu.au} \\
  \AND
  Michael W. Mahoney \\
  ICSI, LBNL, and Department of Statistics \\
  University of California, Berkeley \\
  \texttt{mmahoney@stat.berkeley.edu}
}

\begin{document}

\title{A Heavy-Tailed Algebra for Probabilistic Programming}

\maketitle

\begin{abstract}
Despite the successes of probabilistic models based on passing noise through neural networks, recent work has identified that such methods often fail to capture tail behavior accurately---unless the tails of the base distribution are appropriately calibrated. 
To overcome this deficiency, we propose a systematic approach for analyzing the tails of random variables, and we illustrate how this approach can be used during the static analysis (before drawing samples) pass of a probabilistic programming language compiler. 
To characterize how the tails change under various operations, we develop an algebra which acts on a three-parameter family of tail asymptotics and which is based on the generalized Gamma distribution. 
Our algebraic operations are closed under addition and multiplication; they are capable of distinguishing sub-Gaussians with differing scales; and they handle ratios sufficiently well to reproduce the tails of most important statistical distributions directly from their definitions. 
Our empirical results confirm that inference algorithms that leverage our heavy-tailed algebra attain superior performance across a number of density modeling and variational inference tasks.
\end{abstract}

\section{Introduction}

Within the context of modern probabilistic programming languages (PPLs), recent developments in functional programming \citep{tolpin2016design}, programming languages \citep{bernstein2019static}, and deep variational inference (VI) \citep{bingham2019pyro} combine to facilitate efficient probabilistic modelling and inference. 
Despite the broadening appeal of probabilistic programming, however, common pitfalls such as mismatched distribution supports \citep{lee2019towards} and non-integrable expectations \citep{wang2018variational,vehtari2015pareto,yao2018yes} remain uncomfortably commonplace and remarkably challenging to address. In particular, heavy-tailed distributions arise in a wide range of statistical applications and are known to present substantial technical challenges~\citep{NairWiermanZwart,yao2018yes,wang2018variational}.
Recent innovations aiming to improve PPLs have automated verification of
distribution constraints \citep{lee2019towards}, tamed noisy gradient estimates \citep{eslami2016attend} as well as unruly density ratios
\citep{vehtari2015pareto,wang2018variational}, and approximated high-dimensional distributions with non-trivial bulks \citep{papamakarios2021normalizing}. To address the issue of heavy-tailed targets, approaches which initialize with non-Gaussian tails have been proposed \citep{jaini2020tails,ftvi}. However, these methods typically require the use of optimization and/or sampling strategies to estimate the tails of the target distribution. Such strategies are often unstable, or they fail to allow for a sufficiently wide array of possible tail behaviours. 

Motivated by this, we introduce the first procedure for static analysis of a probabilistic program that automates analysis of target distributions' tails. In addition, we show how tail metadata obtained from this procedure can be leveraged by PPL compilers to generate inference algorithms which mitigate a number of pathologies.
For example, importance sampling estimators can exhibit infinite variance if the tail of the approximating density is lighter than the target; most prominent black-box VI methods are incapable of changing their tail behaviour from an initial proposal distribution \citep{jaini2020tails,ftvi}; and Monte-Carlo Markov Chain (MCMC) algorithms may also lose ergodicity when the tail of the target density falls outside of a particular family \citep{roberts1996exponential}. 
All of these issues could be avoided if the tail of the target is known before runtime.

To classify tail asymptotics, 
we propose a three-parameter family of distributions which is closed under most typical operations.
This family is based on the generalized Gamma distribution (\Cref{eq:GenGammaDensity}), and it interpolates between established asymptotics on sub-Gaussian random variables \citep{ledoux2001concentration} and regularly varying random variables \citep{mikosch}.
Algebraic operations on random variables can then be lifted to computations on the tail parameters.
This results in a \emph{heavy-tailed algebra} that we designate as the \emph{generalized Gamma algebra (GGA)}.
Through analyzing operations like $X + Y$, $X^2$, and $X / Y$ at the level of densities (e.g., additive convolution $p_X \oplus p_Y$), the tail parameters of a target density can be estimated from the parameters of any input distributions using \Cref{tab:gga_operations}.

Operationalizing our GGA, we propose a tail inferential static analysis strategy analogous to traditional type inference. %
GGA tail metadata can be used to diagnose and address tail-related problems in downstream tasks, such as
employing Riemannian-manifold methods \citep{girolami2011riemann} to sample heavy tails
or preemptively detect unbounded expectations.
In this paper, we consider density estimation and VI, where we use the GGA-computed tail of the target density to calibrate our density approximation.
When composed with a learnable Lipschitz pushforward map (\Cref{ssec:repr_dist}),
the resulting combination is a flexible density approximator with tails provably calibrated to match those of the target.

\begin{figure}
\begin{tblr}{colspec={Q[.4\textwidth]Q[.2\textwidth]Q[.3\textwidth]}}
\centering \bf Analyze Target & \centering \bf Calibrate Tails & \centering \bf Refine Bulk
\end{tblr}
\vspace{-1.2cm}

{
\SetTblrInner{rowsep=-8pt}
\begin{tblr}{colspec={c Q[.15cm] c Q[1.1cm] c Q[.15cm] c},columns={m,m,m,m,m,m,m}}
\raisebox{-.5\height}{\includegraphics[width=0.25\textwidth]{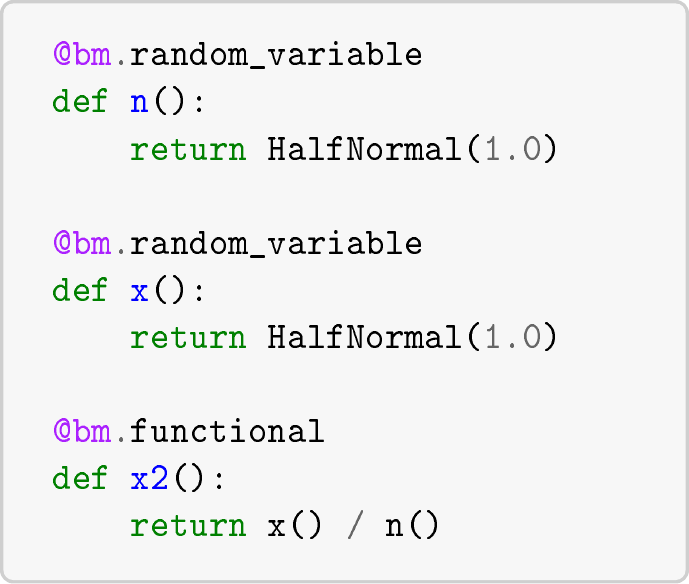} \hspace{-.3cm}}
&\scalebox{1}{$\boldsymbol{\longrightarrow}$} & \raisebox{-.4\height}{\hspace{-.3cm} \includegraphics[width=0.2\textwidth]{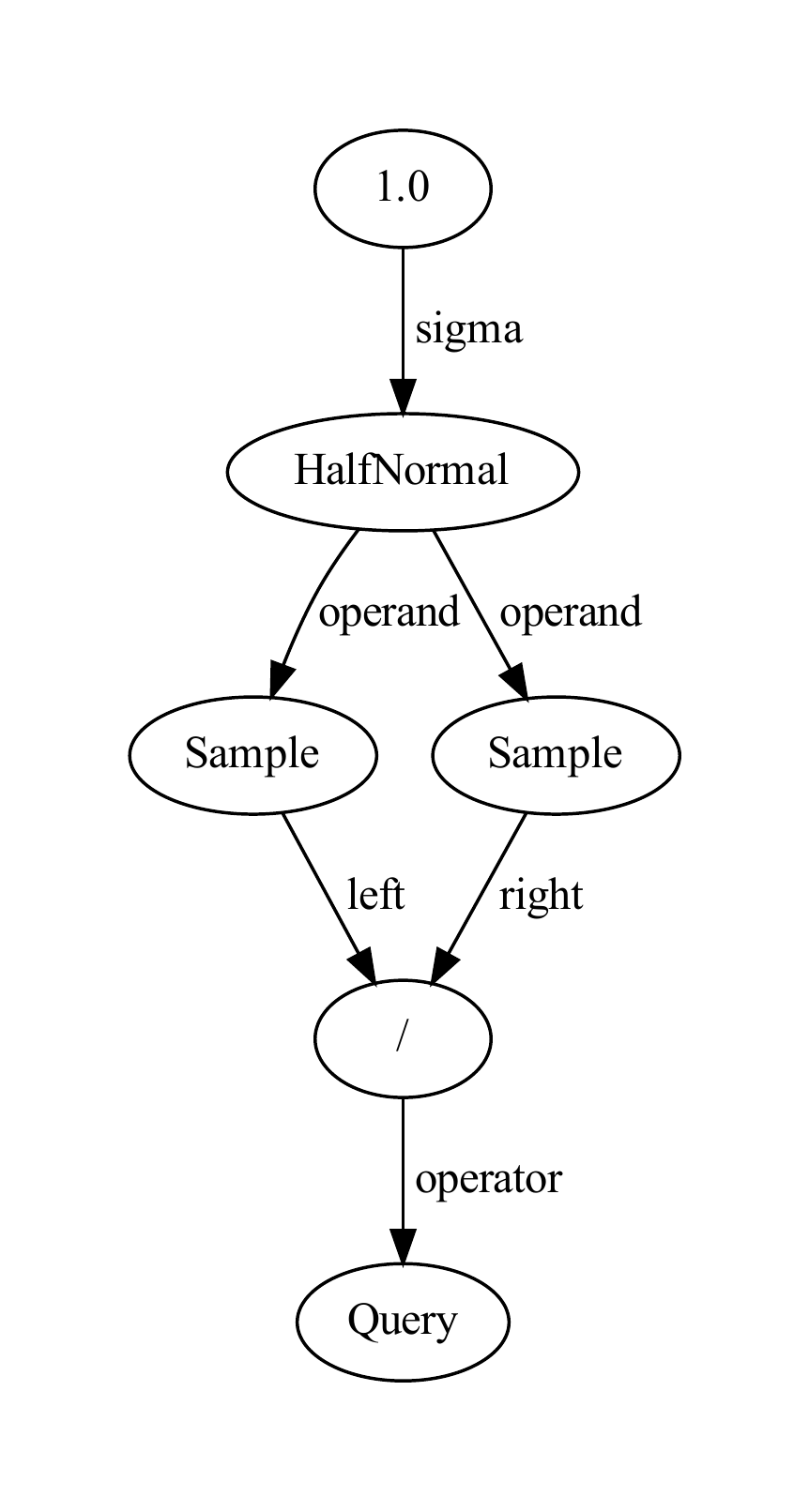}\hspace{-.2cm}} & \hspace{-.5cm}\scalebox{1.7}{$\xrightarrow{\,\,\textit{\textbf{GGA}}\,\,}$}  &  \stackanchor{\hspace{-0.2cm}\includegraphics[width=0.15\textwidth]{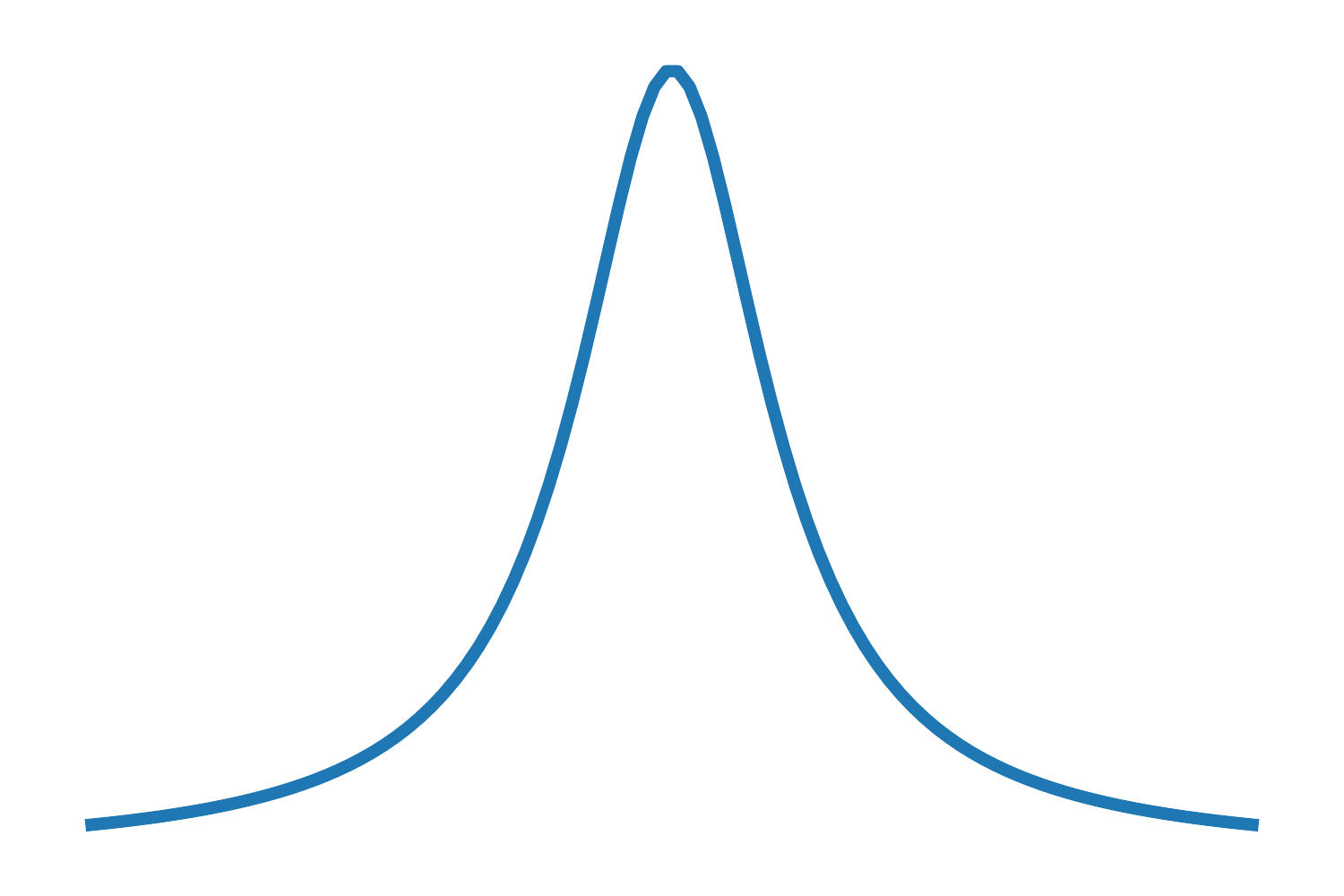}}{\hspace{-0.2cm}\scalebox{1.4}{$\mathcal{R}_2$}}\hspace{-.3cm} & \scalebox{1}{$\boldsymbol{\longrightarrow}$ } & \raisebox{-.5\height}{\includegraphics[width=0.2\textwidth]{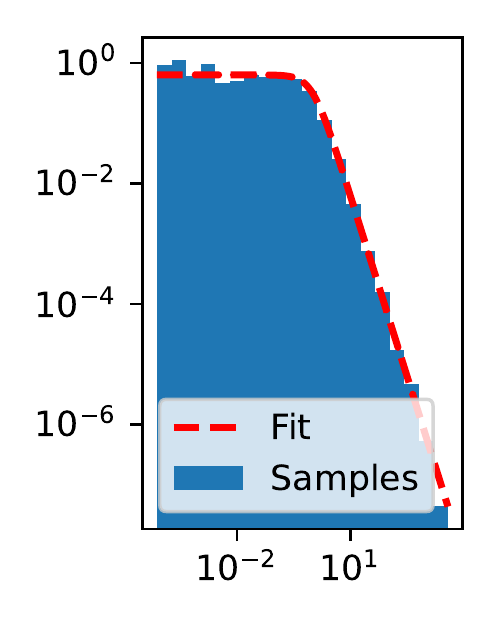}} \\
(1) & & (2) & & (3) & & (4)
\end{tblr}
}
	\caption{Our heavy-tailed algebra ensures that the tails of density estimators and variational approximations are calibrated to those of the target distribution.
    Here, a generative model expressed in a PPL (1) is analyzed using the GGA \emph{without drawing any samples} (2) to compute the tail parameters of the target. 
		A representative distribution with calibrated tails is chosen for the initial approximation (3), and a learnable tail-invariant
		Lipschitz pushforward (see bottom of \Cref{tab:gga_operations}, and \Cref{lem:lipschitz}) is optimized (4) to correct the bulk approximation. 
 }
\end{figure}

\textbf{Contributions.}
Here are our main contributions.

\vspace{-.2cm}
\begin{itemize}[leftmargin=*]
	\item We propose the generalized Gamma algebra (GGA) as an example of a heavy-tailed algebra for probability distributions. 
	This extends prior work on classifying tail asymptotics, and it includes both sub-Gaussian / sub-exponentials \citep{ledoux2001concentration} as well as power-law / Pareto-based tail indices \citep{clauset2009power}. Composing operations outlined in \Cref{tab:gga_operations}, one can compute the GGA tail class for downstream random variables of interest.
	\item We implement the GGA in the static analysis phase of a PPL compiler.
	      This unlocks the ability to leverage GGA metadata in order to
	      better tailor MCMC and VI algorithms produced by a PPL.
	\item We demonstrate that density estimators which combine our GGA tails with neural networks
 (autoregressive normalizing flows \cite{papamakarios2021normalizing} and neural spline flows \cite{durkan2019neural})
	    achieve calibrated tails without sacrificing good bulk approximation.
\end{itemize}

\section{The Generalized Gamma Algebra}\label{sec:gga}

To start, we formulate our heavy-tailed algebra of random variables that is closed under most standard elementary operations (addition, multiplication, powers). The central class of random variables under consideration are those with tails of the form in Definition \ref{def:gg_tail}.

\begin{definition}\label{def:gg_tail}
A random variable $X$ is said to have a \emph{generalized Gamma tail} if the Lebesgue density of $|X|$ satisfies
\begin{equation}
\label{eq:GenGammaTails}
p_{|X|}(x) \sim c x^\nu e^{-\sigma x^\rho}, \qquad \text{as } x \to \infty,
\end{equation}
for some $c > 0$, $\nu \in \mathbb{R}$, $\sigma > 0$ and $\rho \in \mathbb{R}$. Denote the set of all such random variables by $\mathcal{G}$.
\end{definition}
Consider the following equivalence relation on $\mathcal{G}$: $X \equiv Y$ if and only if $0 < p_{|X|}(x) / p_{|Y|}(x) < +\infty$ for all sufficiently large $x$. 
The resulting equivalence classes can be represented by their corresponding parameters $\nu, \sigma, \rho$.
Hence, we denote the class of random variables $X$ satisfying \Cref{eq:GenGammaTails} by $(\nu,\sigma,\rho)$. 
In the special case where $\rho = 0$, for a fixed $\nu < -1$, each class $(\nu,\sigma,0)$ for $\sigma > 0$ is equivalent, and is denoted by $\mathcal{R}_{|\nu|}$, representing \emph{regularly varying} tails. 
Our algebra operates on these equivalence classes of $\mathcal{G}$, characterizing the change in tail behaviour under various operations. To incorporate tails which lie outside of $\mathcal{G}$, we let $\mathcal{R}_1$ incorporate \emph{super-heavy tails}, which denote random variables with tails heavier than any random variable in $\mathcal{G}$. All operations remain consistent with this notation. Likewise, we let $\mathcal{L}$ denote \emph{super-light tails}, which are treated in our algebra as a class where $\rho = +\infty$ (effectively constants).

\Cref{eq:GenGammaTails} and the name of the algebra are derived from the generalized Gamma distribution.

\begin{table*}[h!]
    \centering
    \bgroup
\SetTblrInner{rowsep=7pt}
\begin{tblr}{colspec={||Q[2.35cm]|Q[c,1.8cm]X||},columns={m,t,m},hlines}
\hline
\centering \textbf{Ordering} & 
$(\nu_1,\sigma_1,\rho_1) \leq $ $\quad(\nu_2,\sigma_2,\rho_2) $ & $\displaystyle{\iff \limsup_{x\to\infty} \frac{x^{\nu_1} e^{-\sigma_1 x^{\rho_1}}}{x^{\nu_2} e^{-\sigma_2 x^{\rho_2}}} < +\infty.}$ \\ 
\centering \textbf{Addition} & $(\nu_{1},\sigma_{1},\rho_{1})$ $\oplus$ \newline $(\nu_{2},\sigma_{2},\rho_{2})$ & \scalebox{0.95}{$\equiv\begin{cases}
\max\{(\nu_{1},\sigma_{1},\rho_{1}),(\nu_{2},\sigma_{2},\rho_{2})\} & \text{if }\rho_{1}\neq\rho_{2}\text{ or }\rho_{1},\rho_{2}<1\\[2pt]
(\nu_{1}+\nu_{2}+1,\min\{\sigma_{1},\sigma_{2}\},1) & \text{if }\rho_{1}=\rho_{2}=1\\[5pt]
\displaystyle{\bigg(\nu_{1}+\nu_{2}+\frac{2-\rho}{2},\big(\sigma_{1}^{-\tfrac{1}{\rho-1}}+\sigma_{2}^{-\tfrac{1}{\rho-1}}\big)^{1-\rho},\rho\bigg)} & \text{if }\rho=\rho_{1}=\rho_{2}>1.
\end{cases}$}
\\ 
\centering \textbf{Powers} & $(\nu,\sigma,\rho)^\beta$& $\displaystyle{\equiv \left(\frac{\nu+1}{\beta} - 1,\sigma,\frac{\rho}{\beta}\right)}$ for $\beta > 0$ \\
\centering \textbf{Reciprocal*}${}^\dagger$ & $(\nu,\sigma,\rho)^{-1}$ & $\equiv \begin{cases} (-\nu-2,\sigma,-\rho) & \text{ if } (\nu + 1)/\rho > 0 \text{ and } \rho \neq 0 \\
\mathcal{R}_2 & \text{ otherwise}
\end{cases}$ \\ 
\centering \textbf{Scalar \linebreak Multiplication} & $c(\nu,\sigma,\rho)$ & $\equiv (\nu,\sigma |c|^{-\rho},\rho)$\\
\centering \textbf{Multiplication*} & \vspace{-2cm} \newline $(\nu_{1},\sigma_{1},\rho_{1})\otimes(\nu_{2},\sigma_{2},\rho_{2})$&$  \equiv\begin{cases}
\displaystyle{\left(\frac{1}{\mu}\left(\frac{\nu_{1}}{|\rho_{1}|}+\frac{\nu_{2}}{|\rho_{2}|}+\frac{1}{2}\right),\sigma,-\frac{1}{\mu}\right)} & \text{ if }\rho_{1},\rho_{2}<0\\[15pt]
\displaystyle{\left(\frac{1}{\mu}\left(\frac{\nu_{1}}{\rho_{1}}+\frac{\nu_{2}}{\rho_{2}}-\frac{1}{2}\right),\sigma,\frac{1}{\mu}\right)} & \text{ if }\rho_{1},\rho_{2}>0\\[10pt]
\mathcal{R}_{|\nu_1|} & \mbox{ if }\rho_{1}\leq0,\rho_{2}>0 \\[5pt]
\mathcal{R}_{\min\{|\nu_1|,|\nu_2|\}} & \mbox{ if }\rho_{1}=0,\rho_{2}=0
\end{cases}$ \newline \vspace{.3cm} \newline  
$\displaystyle{\begin{aligned}
\text{where }\mu&=\frac{1}{|\rho_{1}|}+\frac{1}{|\rho_{2}|}=\frac{|\rho_{1}|+|\rho_{2}|}{|\rho_{1}\rho_{2}|}\\
\sigma&=\mu(\sigma_{1}|\rho_{1}|)^{\tfrac{1}{\mu|\rho_{1}|}}(\sigma_{2}|\rho_{2}|)^{\tfrac{1}{\mu|\rho_{2}|}}.
\end{aligned}}$\\ 
\centering \textbf{Product of Densities*} & $(\nu_1,\sigma_1,\rho_1)\,\&$\newline $(\nu_2,\sigma_2,\rho_2) $&$
\equiv \begin{cases}
(\nu_{1}+\nu_{2},\sigma_{1},\rho_{1}) & \text{ if }\rho_{1}<\rho_{2}\\
(\nu_{1}+\nu_{2},\sigma_{1}+\sigma_{2},\rho) & \text{ if }\rho=\rho_{1}=\rho_{2}\\
(\nu_{1}+\nu_{2},\sigma_{2},\rho_{2}) & \text{ otherwise.}
\end{cases}$ \\ 
\centering \textbf{Exponentials*}${}^\dagger$ & $\exp(\nu,\sigma,\rho)$ & $\equiv \begin{cases} \mathcal{R}_{\sigma+1} & \text{if } \rho \geq 1 \\ \mathcal{R}_1 & \text{ otherwise.} \end{cases}$ \\
\centering \textbf{Logarithms*}${}^\dagger$ & $\log(\nu,\sigma,\rho)$ & $\equiv \begin{cases} (0, |\nu|-1, 1) & \text{if } \nu < -1 \\ \mathcal{L} & \text{ otherwise.} \end{cases}$ \\
\centering \textbf{Functions ($L$-Lipschitz)} & 
$f(X_1,\dots,X_n) $&$\equiv L \max\{X_1,\dots,X_n\}$ \\ 
\hline
    \end{tblr}%
    \vspace{.25cm}
    \caption{\label{tab:gga_operations}\textbf{\textit{The Generalized Gamma Algebra.}} Operations on random variables (e.g., $X_1 + X_2$) are 
    viewed as actions on density functions
    (e.g., convolution $(\nu_1, \sigma_1, \rho_1)\oplus (\nu_2, \sigma_2, \rho_2)$) and the tail parameters of the result are analyzed and reported. In this table, * denotes novel results, and $\dagger$ denotes that additional assumptions are required.}
\egroup
\end{table*}

\begin{definition}
Let $\nu \in \mathbb{R}$, $\sigma > 0$, and $\rho \in \mathbb{R} \backslash \{0\}$ be such that $(\nu+1)/ \rho > 0$.
A non-negative random variable $X$ is \emph{generalized Gamma distributed} with parameters $\nu,\sigma,\rho$ if it has Lebesgue density
\begin{equation}    
\label{eq:GenGammaDensity}
p_{\nu,\sigma,\rho}(x) = c_{\nu,\sigma,\rho} x^\nu e^{-\sigma x^\rho},\qquad x > 0,
\end{equation}
where $c_{\nu,\sigma,\rho} = \rho \sigma^{(\nu+1)/\rho} / \Gamma((\nu+1)/\rho)$ is the normalizing constant. %
\end{definition}
The importance of the generalized Gamma form arises due to a combination of two factors:
\begin{enumerate}[label={(\roman*)},leftmargin=*]
    \item 
    The majority of interesting continuous univariate distributions with infinite support satisfy \Cref{eq:GenGammaTails}, including
    Gaussians ($\nu=0$, $\rho=2$),
    gamma/exponential/chi-squared ($\nu > -1$, $\rho=1$), Weibull/Frechet ($\rho = \nu + 1$), and
    Student $T$/Cauchy/Pareto ($\mathcal{R}_\nu$).
    A notable exception is the log-normal distribution (see \Cref{eg:log-normal} in \Cref{sec:addtl_eg}).
    \item 
    The set $\mathcal{G}$ is known to be closed under additive convolution, positive powers, and Lipschitz functions. 
    We prove it is closed under multiplicative convolution as well. 
    This covers the majority of elementary operations on independent random variables. Reciprocals, exponentials and logarithms comprise the only exceptions;
    however, we will introduce a few ``tricks'' to handle these cases as well. 
\end{enumerate}
The full list of operations in GGA is compiled in \Cref{tab:gga_operations} and is described in detail in \Cref{sec:gga_operations}.
GGA classes for common probability distributions are provided in \Cref{sec:univariate_classes}.
All operations in the GGA can be proven to exhibit identical behaviour with their corresponding operations on random variables, with the sole exception of reciprocals (marked by $\dagger$ in \Cref{tab:gga_operations}), where additional assumptions are required. The asymptotics for operations marked with an asterisk are novel to this work. For further details, see \Cref{sec:gga_operations}.

\paragraph{Posterior distributions.}

A primary application of PPLs is to perform Bayesian inference.
To cover this use case, it is necessary to prescribe a procedure to deal with \emph{posterior distributions}. 
Consider a setup where a collection of random variables $X_1,\dots,X_n$ are dependent on corresponding latent random elements $Z_1,\dots,Z_n$ as well as a parameter $\theta$ through functions $f_i$ by $X_i = f_i(\theta, Z_i)$. 
For simplicity, we assume that each $f_i = f_{i,k}\circ f_{i,k-1}\circ\cdots \circ f_{i,1}$ where each $f_{ij}$ is an elementary operation in Table \ref{tab:gga_operations}. 
To estimate the tail behaviour of $\theta$ conditioned on $X$, we propose an elementary approach involving inverses. 
For each operation $f_{ij}$, if $f_{ij}$ is a power, reciprocal, or multiplication operation, let $R_{ij}$ be given according to the following:

\centering \begin{tabular}{rll}
\textbf{Powers:}&$f(x) = x^\beta$,&$R \equiv (1-\beta,1,0)$\\
\textbf{Reciprocals:}&$f(x)=x^{-1}$,&$R \equiv (2,1,0)$\\
\textbf{Multiplication:}&$f(x,y)=xy$,&$R\equiv (1,1,0)$
\end{tabular}
and otherwise, let $R_{ij} \equiv 1$.

\justifying Letting $f^{-1}_i(x,z)$ denote the inverse of $f_i$ in the first argument, we show in \Cref{sec:gga_operations} that
\vspace{-.15cm}
\[
\theta \vert \boldsymbol{X} = \boldsymbol{x} \equiv 
\bigg(\amper_{i=1}^n f^{-1}_i(\boldsymbol{x}, Z_i)\bigg)\&\bigg(\amper_{i,j=1}^{n,k} R_{ij}\bigg)\, \&\, \pi,
\]
where $\pi$ denotes the prior for $\theta$ and $X \& Y$ denotes the product of densities operation. Since the inverse of a composition of operations is a composition of inverses, the tail of $f_i^{-1}(\boldsymbol{x},Z_i)$ can be determined by backpropagating through the computation graph for $X_i$ and sequentially applying inverse operations. Consequently, the tail behaviour of the posterior distribution for one parameter can be obtained using a single backward pass. Posterior distributions for multiple parameters involve repeating this procedure one parameter at a time, with other parameters fixed.

\section{Implementation}\label{sec:impl}

\subsection{Compile-time static analysis}

\begin{minipage}[t]{.5\textwidth}
To illustrate an implementation of GGA for static analysis, we sketch the operation of the PPL compiler at a high-level.
A probabilistic program is first inspected using Python's built-in \texttt{ast} module
and transformed to static single assignment (SSA) form \citep{rosen1988global}.
Next, standard compiler optimizations (e.g., dead code elimination, constant propagation)
are applied and an execution of the optimized program is traced \citep{wingate2011lightweight,bingham2019pyro}
and accumulated in a directed acyclic graph representation. A breadth-first type checking pass, as seen in Algorithm~\ref{alg:bfs_typecheck}, completes in linear time, and GGA results may be applied to implement \texttt{computeGGA()} using the following steps:
\end{minipage}
\hspace{.03\textwidth} 
\begin{minipage}[t]{.45\textwidth}
\vspace{-.75cm}
\SetKwComment{Comment}{/* }{ */}
\begin{algorithm}[H]
	\caption{%
	GGA tails static analysis pass}\label{alg:bfs_typecheck}
	\KwData{Abstract syntax tree for a PPL program}
	\KwResult{GGA parameter estimates for all random variables}
	frontier $\gets$ [rv : Parents(rv) = $\emptyset$]\;
	tails $\gets \{\}$\;
	\While{\text{frontier} $\neq \emptyset$}{
		next $\gets$ frontier.popLeft()\;
		tails[next] $\gets$ computeGGA(next.op, next.parent)\;\hspace{-.1cm}
		frontier $\gets$ frontier + next.children()\;
	}
	\Return{tails}
\end{algorithm}
\end{minipage}
\begin{itemize}[leftmargin=*]
	\item If a node has no parents, then it is an atomic distribution and its tail parameters are known (\Cref{tab:dist_list});
 	\item Otherwise, the node is an operation taking its potentially stochastic inputs (parents) to its output; in which case, consult \Cref{tab:gga_operations} for the output GGA tails.
\end{itemize}

\subsection{Representative distributions}
\label{ssec:repr_dist}

For each $(\nu,\sigma,\rho)$, we make a carefully defined choice of $p$ on $\mathbb{R}$ such that if $X \sim p$, then $X \equiv (\nu,\sigma,\rho)$. This way, any random variable $f(X)$, where $f$ is $1$-Lipschitz, will exhibit the correct tail, and so approximations of this form may be used for VI or density estimation. Let $X \equiv (\nu,\sigma,\rho)$ and $0 < \epsilon \ll 1$ denote a small parameter such that tails $e^{-x^\epsilon}$ are deemed to be ``very heavy'' (we chose $\epsilon = 0.1$). 
Our candidate distributions are as follows.

\begin{itemize}
    \item[($\rho \leq 0$)] If $\rho \leq -1$, then $p_X(x) \sim c x^{-|\nu|}$. One such density is the \emph{Student $t$ distribution}, in this case, with $|\nu|-1$ degrees of freedom if $\nu < -1$ (generate $X \sim \text{StudentT}(|\nu|-1))$.
    \item[($\rho > \epsilon$)] For moderately sized $\rho > 0$, the symmetrization of the generalized Gamma density (\ref{eq:GenGammaDensity}). 
    \item[($\rho \leq \epsilon$)] If $X \equiv (\nu,\sigma,\rho)$ where $\rho$ is small, then $X$ will exhibit much heavier tails, and the generalized Gamma distribution in Case 1 will become challenging to sample from. In these cases, we expect that the tail of $X$ should be well represented by a power law. 
    The generalized Gamma density (\Cref{eq:GenGammaDensity}) satisfies $\mathbb{E}X^r = \sigma^{-r/\rho} \Gamma(\frac{\nu+1+r}{\rho})/\Gamma(\frac{\nu+1}{\rho})$ for $r > 0$. Let $\alpha > 0$ be such that $\mathbb{E}X^\alpha = 2$. By Markov's inequality, the tail of $X$ satisfies $\mathbb{P}(X>x)\leq 2 x^{-\alpha}$. Therefore, we can represent tails of this form by the Student $t$ distribution with $\alpha+1$ degrees of freedom (generate $X \sim \text{StudentT}(\alpha)$).
\end{itemize}

\subsection{Bulk correction by Lipschitz mapping}

While a representative distribution will exhibit the desired tails,
the target distribution's bulk may be very different from a generalized Gamma and thus result in poor distributional approximation. 
To address this, we propose splicing together the tails from a generalized Gamma with a flexible density approximation for the bulk.
Many combinations are possible.
In this work, we rely on the Lipschitz operation in the GGA (\Cref{lem:lipschitz}) and post-compose neural
spline flows \citep{durkan2019neural} (which are identity functions outside of a bounded interval) after
properly initialized generalized Gamma distributions.
Optimizing the parameters of the flow results in good bulk approximation while simultaneously
preserving the tail correctness guarantees attained by the GGA.

\begin{figure}[h]
    \centering
    
    \begin{tikzpicture}
  \node at (0,0) {\includegraphics[width=0.9\textwidth]{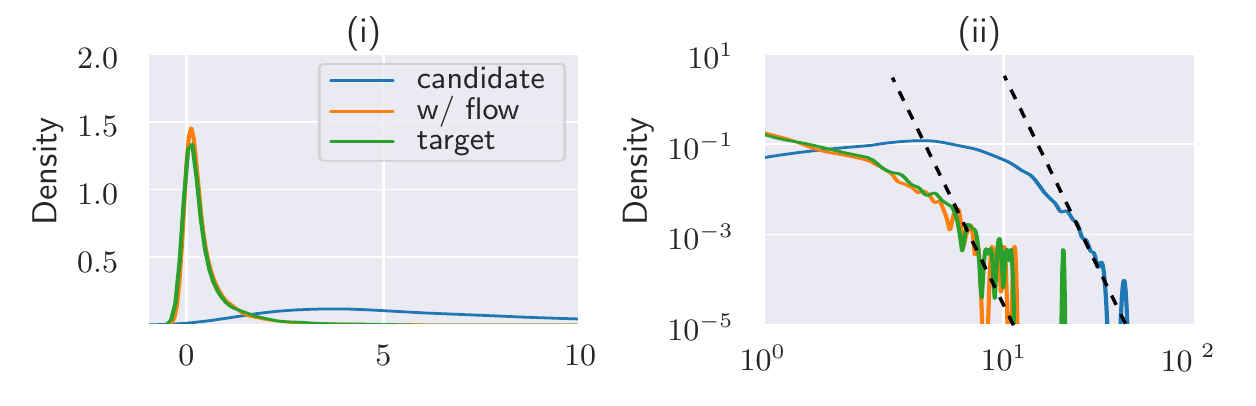}};
  
  \draw[dashed, line width=1pt] (3.25,1) -- (4.5,-1.25);
  \draw[dashed, line width=1pt] (4.6,1) -- (5.85,-1.25);
    \end{tikzpicture}

\vspace{-.25cm}
    \caption{
    The \textcolor{blue}{\textbf{candidate}} distribution chosen by the GGA calibrates tails to the \textcolor{OliveGreen}{\textbf{target}}, but with incorrect bulk. 
    A Lipschitz normalizing \textcolor{orange}{\textbf{flow}} corrects the bulk (i) without changing the tail behaviour, as seen by the parallel tail asymptotics (black dashed lines) in (ii).
}
    \label{fig:power_normal}
\end{figure}

\begin{example}
	Let $A \in \mathbb{R}^{k \times k}$, $x,y\in\mathbb{R}^k$, with $x_i,y_i,A_{ij} \overset{\text{iid}}{\sim} \mathcal{N}(-1,1)$.
	The distribution of $x^\top A y = \sum_{i,j} x_i A_{ij} y_j$ is a convolution of normal-powers \citep{gupta2008analyzing} and lacks a
	closed form expression.
	Using the GGA (\Cref{tab:gga_operations}), one can compute
	its tail parameters to be $(\frac{k}{2}-1,\frac{3}{2}, \frac{2}{3})$. The candidate given by the GGA representative distribution (\Cref{ssec:repr_dist}) is a gamma distribution with correct tail behaviour, but it is a poor approximation otherwise. A learnable Lipschitz bijection is optimized to correct the bulk approximation (\Cref{fig:power_normal}(i)). From the Lipschitz property, the slope of the tail asymptotics in log-log scale remains the same before and after applying the flow correction (\Cref{fig:power_normal}(ii)): the tails are \emph{guaranteed} to remain calibrated.
\end{example}

\begin{example}
    Consider $\sum_{i=1}^4 X_i^2$ where $X_i \sim \text{StudentT}(i)$.
    While we are not aware of a closed-form expression for the density,
    this example is within the scope of our GGA. Empirical results illustrate that our method (\Cref{fig:squared_studentt_qq}(i)) accurately models \emph{both} the bulk and the tail, while Gaussian-based Lipschitz flows (\Cref{fig:squared_studentt_qq}(ii)) inappropriately impose tails which decay too rapidly.
\end{example}

\begin{figure}
    \centering
    \includegraphics[width=0.45\textwidth]{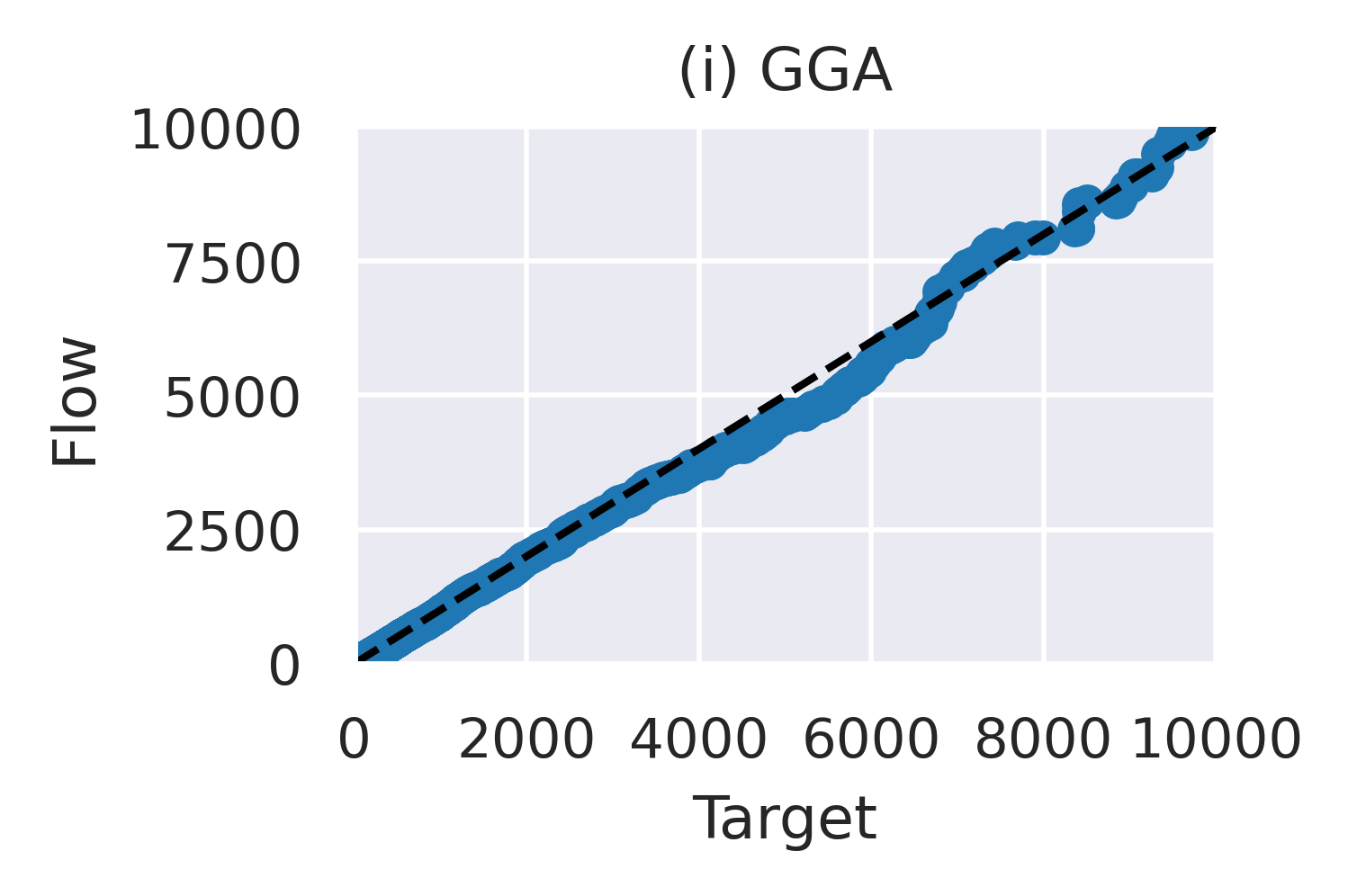}
    \includegraphics[width=0.45\textwidth]{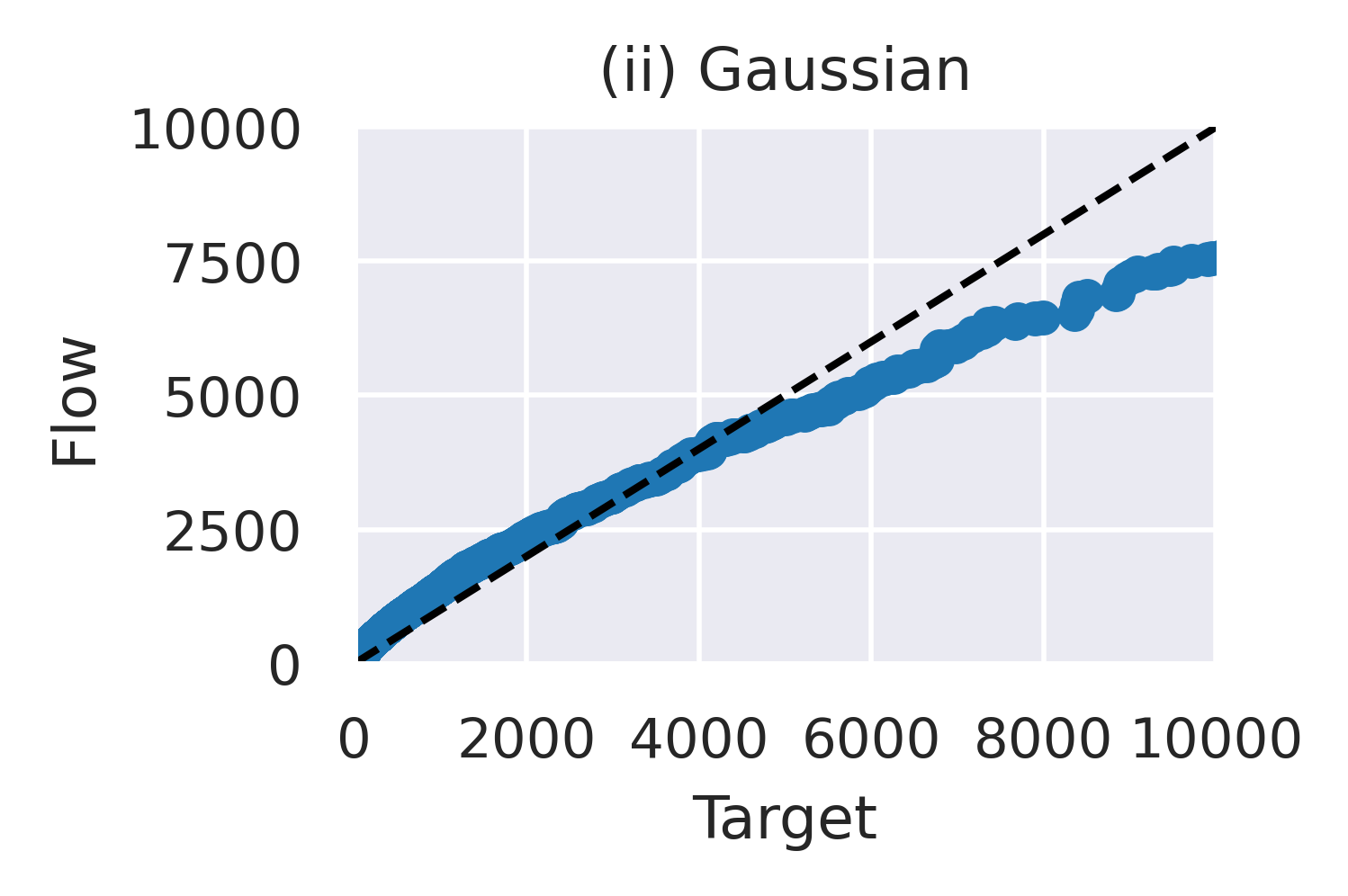}
    \caption{Q-Q plots of density approximations of a heavy-tailed target ($\sum_{i=1}^4 X_i^2$ where $X_i \sim \text{StudentT}(i)$) initialized by our GGA candidate (i) and the Gaussian distribution (ii).
    While the expressive modeling capability of flows enables good approximation of the distribution bulk, Lipschitz transformations of Gaussians inevitably impose miscalibrated squared exponential tails which are not sufficiently heavy as evidenced in (ii).
    \vspace{-.25cm}}
    \label{fig:squared_studentt_qq}
\end{figure}

\section{Theoretical Examples}\label{sec:addtl_eg}

To verify that our GGA yields accurate predictions of tail behaviour, we work
out some explicit GGA computations on several standard distributions using operations in \Cref{tab:gga_operations}.
By doing so, we recover some common probability identities.

~

\begin{example}[Chi-squared random variables]
	Let $X_1,\dots,X_k$ be $k$ independent standard normal random variables. The variable $Z = \sum_{i=1}^k X_i^2$ is \emph{chi-squared distributed} with $k$ degrees of freedom. Using the GGA, we can accurately determine the tail behaviour of this random variable directly from its construction. Recall that each $X_i \equiv (0,1/2,2)$, and by the power operation, $X_i^2 \equiv (-1/2,1/2,1)$. Applying the addition operation $k$ times reveals that $Z \equiv (k/2-1,1/2,1)$ and implies that the density of $Z$ is asymptotically $c x^{k/2-1} e^{-x / 2}$ as $x \to \infty$.
 In fact, it is known that the density of $Z$ is exactly $p_Z(x) = c_k x^{k/2-1} e^{-x/2}$, where $c_k = 2^{-k/2} / \Gamma(k/2)$.
\end{example}
~

\begin{example}[Products of random variables]
	To demonstrate the multiplication operation in our algebra, we consider the product of two exponential, Gaussian, and reciprocal Gaussian random variables. Traditionally, asymptotics for the distribution of the product of two random variables would be found analytically. For example, consider the following Lemma \ref{lem:Products}.
	\begin{lemma}
	\label{lem:Products}
		Let $X_1,X_2 \sim \mbox{Exp}(\lambda)$ and $Z_1,Z_2 \sim \mathcal{N}(0,1)$ be independent. As $x \to \infty$, the densities of $X_1 X_2$, $Z_1 Z_2$ and $Z = 1/Z_1 \cdot 1/Z_2$ satisfy 
		\[
			p_{X_1 X_2}(x) \sim \frac{\lambda^{3/2}\sqrt{\pi}}{x^{1/4}} e^{-2\lambda \sqrt{x}},\quad p_{Z_1 Z_2}(x) \sim \frac{1}{\sqrt{2\pi x}} e^{-x},\quad
			p_{Z}(x) \sim \frac{1}{\sqrt{2\pi}|z|^{3/2}}e^{-1/|z|}.
		\]
	\end{lemma}
 With ease, our algebra correctly determines that $X_1 X_2 \equiv (-\frac14,2\lambda,\frac12)$, $Z_1 Z_2 \equiv (-\frac12, 1, 1)$ and $Z \equiv (-\frac32,1,-1)$. 
 To demonstrate how one would ascertain these asymptotics manually, see the proof of Lemma \ref{lem:Products} in Appendix 
 \ref{sxn:proof_of_a_lemma}.
\end{example}
~

\begin{example}[Reciprocal distributions]
	Perhaps the most significant challenge with our tail algebra is correctly identifying the tail behaviour of reciprocal distributions. Here, we test the efficacy of our formulation with known reciprocal distributions.
	\begin{itemize}
		\item \emph{Reciprocal normal:} $X \sim \mathcal{N}(0,1) \equiv (0,1/2,2)$, and $X^{-1} \equiv (-2,1/2,-2)$.
		\item \emph{Inverse exponential:} $X \sim \text{Exp}(\lambda) \equiv (0,\lambda,1)$, and $X^{-1} \equiv (-2,\lambda,-1)$.
		\item \emph{Inverse $t$-distribution:} $X \equiv \mathcal{R}_\nu$, and $X^{-1} \equiv \mathcal{R}_2$.
		\item \emph{Inverse Cauchy:} $X \equiv \mathcal{R}_2$, it is known $X^{-1}$
		      has the same distribution, and our theory predicts $X^{-1} \equiv \mathcal{R}_2$.
	\end{itemize}
\end{example}
~

\begin{example}[Cauchy distribution]
	A simple special case of the Student $T$ distribution is the Cauchy distribution, which arises as the ratio of two standard normal random variables. For $X\sim\mathcal{N}(0,1)$, $X \equiv (0,1/2,2)$ and $X^{-1} \equiv (-2,1/2,-2)$. %
	Hence, the multiplication operation correctly predicts that the ratio of two standard normal random variables is in $\mathcal{R}_2$.
\end{example}
~

\begin{example}[Student $T$ distribution]
	Let $X$ be a standard normal random variable, and $V$ a chi-squared random variable with $\nu$ degrees of freedom. The random variable $T = X / \sqrt{V / \nu}$ is \emph{$t$-distributed} with $\nu$ degrees of freedom. Since $V \equiv (\nu/2-1,1/2,1)$, multiplying by the constant $1/\nu$ reveals $V / \nu \equiv (\nu/2-1,1/(2\nu),1)$. Applying the square root operation, $\sqrt{V / \nu} \equiv (\nu-1,1/(2\nu),2)$. To compute the division operation, we first take the reciprocal to find $(V/\nu)^{-1/2} \equiv (-\nu-1,1/(2\nu),-2)$. Finally, since $\rho = -2 < 1$ for this random variable, the multiplication operation with $X \equiv (0,1/2,2)$ yields $T \equiv \mathcal{R}_{\nu+1}$.  Thus, the density of $T$ is asymptotically $c x^{-\nu-1}$ as $x \to \infty$.
	In fact, it is known that the density of $T$ satisfies $p_T(x) = c_\nu (1 + x^2 / \nu)^{-(\nu+1)/2}$ where $c_\nu = \Gamma(\frac{\nu+1}{2})/\Gamma(\frac{\nu}{2}) (\nu\pi)^{-1/2}$, which exhibits the predicted tail behaviour.
\end{example}
~

\begin{example}[Log-normal distribution]\label{eg:log-normal}
	Although the log-normal distribution does not lie in $\mathcal{G}$, the existence of log-normal tails arising from the multiplicative central limit theorem is suggested by our algebra. Let $X_1,X_2,\dots$ be independent standard normal random variables, and let $Z_k = X_1\cdots X_{2^k}$ for each $k=1,2,\dots$. By the multiplicative central limit theorem, letting $\tau = \exp(\mathbb{E}\log |X_i|) \approx 1.13$, we have that
 $$\left(\frac{X_{1}\cdots X_{n}}{\tau}\right)^{1/\sqrt{n}} \overset{\mathcal{D}}{\longrightarrow} Z\qquad\text{as }n \to \infty,$$ where $Z$ is a log-normal random variable with density
	\[
		p_Z(x) = \frac{1}{x\sqrt{2\pi}} \exp(-\tfrac12 (\log x)^2).
	\]
	Therefore, the same is true for $V_k = (Z_k / \tau)^{2^{-k/2}}$. Using our algebra, we will attempt to reproduce the tail of this density. Letting $\tilde{Z}_k = X_{2^k} \cdots X_{2^{k+1}}$, we see that $Z_{k+1} = Z_k \tilde{Z}_k$, and $Z_k,\tilde{Z}_k$ are iid. Let $Z_k \equiv (\nu_k, \sigma_k, \rho_k)$, by induction using the multiplication operation, we find that 
 \begin{align*}
 \nu_{k+1} &=\frac{1}{\mu}\left(\frac{2\nu_{k}}{\rho_{k}}-\frac{1}{2}\right)=\nu_{k}-\frac{\rho_{k}}{4}\\
 \sigma_{k+1}&=\mu\left(\sigma_{k}\rho_{k}\right)^{\frac{2}{\mu\rho_{k}}}=\frac{2}{\rho_{k}}\left(\sigma_{k}\rho_{k}\right)=2\sigma_{k}\\
 \rho_{k+1}&=\frac{1}{\mu}=\frac{\rho_{k}}{2}.
 \end{align*}
 Since $\rho_0 = 2$, $\sigma_0 = 1/2$, and $\nu_0 = 0$, we find that $\rho_k = 2^{1-k}$ and $\sigma_k = 2^{k-1}$. Furthermore, $\nu_{k+1} = \nu_k - 2^{-k-1}$ and so $\nu_k = -1 + 2^{-k}$. %
 Therefore 
\begin{align*}
    Z_k &\equiv (-1+2^{-k},2^{k-1},2^{1-k}),\qquad\text{and,} \\
    V_k &\equiv (-1+2^{-k/2},2^{k-1}\tau^{-2^{1-k}},2^{1-k/2}),
\end{align*}
and letting $\epsilon_k = 2^{-k/2}$, the tail behaviour of the density of $V_k$ satisfies
\begin{align*}
p_k(x) & \sim c_k x^{-1+\epsilon_{k}}\exp\left(-\frac{\epsilon_{k}^{-2}}{2\tau^{-2\epsilon_{k}^{2}}}x^{2\epsilon_k}\right) \\& \sim c_k x^{-1+\epsilon_{k}}\exp\left(-\frac{1}{2\tau^{-2\epsilon_{k}^{2}}}\left(\frac{x^{\epsilon_{k}}-1}{\epsilon_{k}}\right)^{2}\right) \approx c_k x^{-1}\exp\left(-\frac{1}{2}(\log x)^2\right),
\end{align*}
as $x \to \infty$, where the approximation improves as $k$ gets larger. The quality of this approximation is demonstrated in \Cref{fig:lognormal}.
\end{example}

\begin{figure*}[h]
	\centering
	\includegraphics[width=0.48\textwidth]{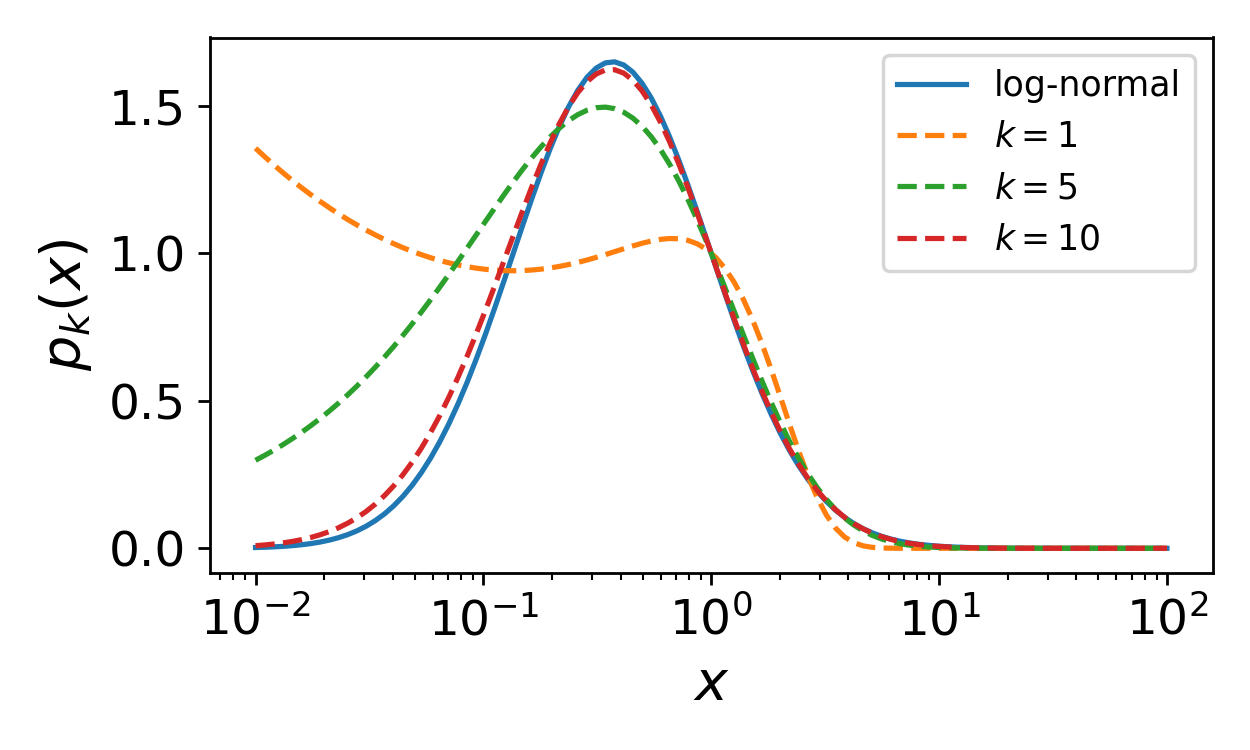}
	\includegraphics[width=0.48\textwidth]{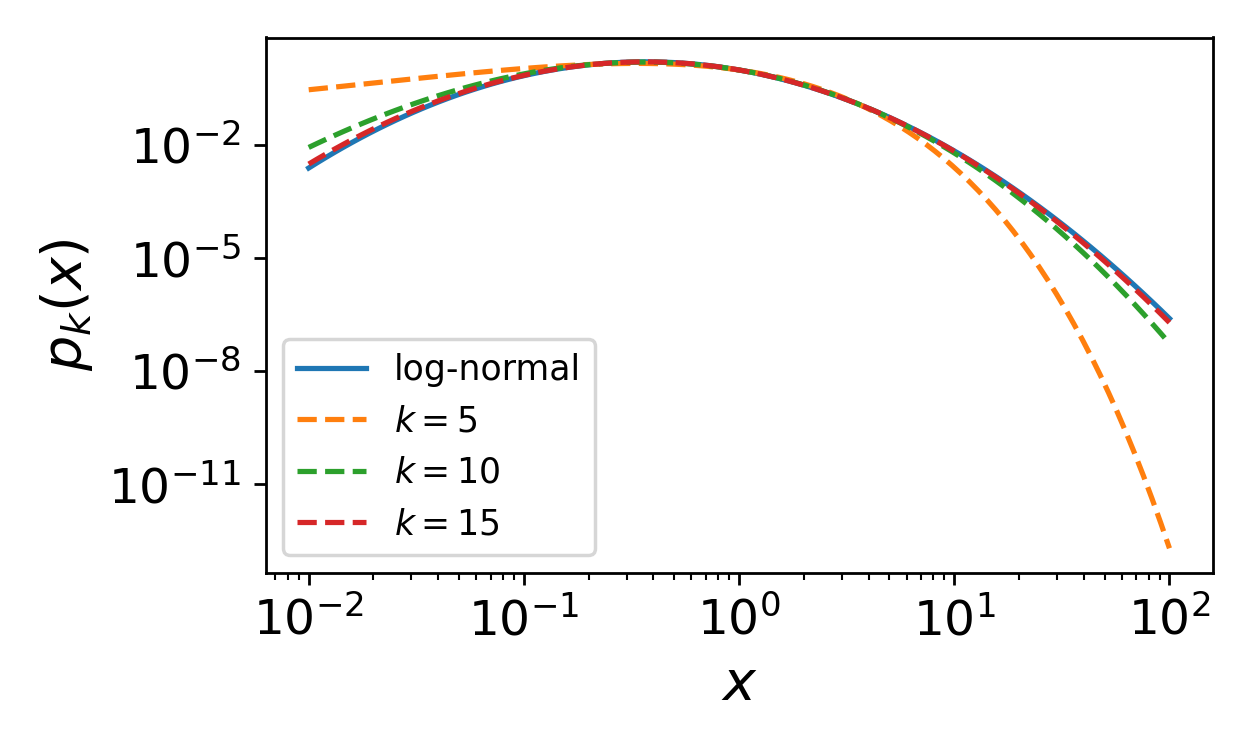}
	\caption{Estimation of the log-normal density (blue) by the representative density chosen by the GGA applied to $V_k$ for $k=1,5,10$ (orange, green, red, respectively), as presented on a log-linear scale (left) and a log-log scale (right). }
	\label{fig:lognormal}
\end{figure*}

\section{Empirical Results}
\label{sec:experiments}

We now demonstrate that GGA-based density estimation yields improvements in tail estimation across several metrics. Our experiments consider normalizing flows initialized from (i) the parametric family defined in \Cref{ssec:repr_dist} against (ii) a normal distribution (status quo).
To further contrast the individual effect of using a GGA base distribution over standard normals against more expressive pushforward maps \citep{durkan2019neural}, we also report ablation results where normalizing flows are replaced by affine transforms, as originally proposed in \citep{kucukelbir2017automatic}. 
All experiments are repeated for 100 trials, trained to convergence using the Adam optimizer with manually tuned learning rate. Additional details are available in \Cref{sec:experiment_details}.
All target distributions in this section are expressed as generative PPL programs:
Cauchy using a reciprocal normal; 
Chi2 (chi-squared) using a sum of squared normals; 
IG (Inverse Gamma) using a reciprocal exponential; 
normal using a sum of normals; and 
StudentT using a normal and Cauchy ratio. 
Doing so tasks the static analyzer to infer the target's tails and makes the analysis non-trivial.

Our results in the following tables 
share a consistent narrative: 
\emph{a GGA base distribution rarely hurts, and it can significantly help with heavy tailed targets}.
Importantly, standard evaluation metrics such as negative cross-entropy, ELBO, or importance-weighted autoencoder bounds \citep{burda2015importance} do \emph{not} evaluate the quality of tail approximations. Instead, we consider diagnostics which do evaluate the quality of tail approximations: namely, an estimated tail exponent $\hat{\alpha}$, and the Pareto $\hat{k}$ diagnostic
\citep{yao2018yes}.
Except for when targets are truly light tailed ($\alpha=\infty$ in Chi2 and normal),
GGA-based approximations are the only ones to reproduce appropriate GPD tail index $\hat\alpha$ in density estimation and achieve a passing $\hat{k}$ below $0.2$ in VI.
Less surprising is  that adding a flow improved approximation metrics, as we expect the additional representation flexibility to be beneficial.

\paragraph{Density Estimation.} 
Given samples $\{x_i\}_{i=1}^N$ from a target density $p$, we minimize a Monte-Carlo estimate of the cross entropy $H(p,q) = -E_p[\log q(X)] \approx -\frac{1}{N} \sum_{i=1}^N \log q(x_i)$.
The results are shown in \Cref{tab:de} and \Cref{tab:de_ll} along with power-law tail index estimates $\hat\alpha$ \citep{clauset2009power}.
Closeness between the target Pareto tail index $\alpha$ \citep{clauset2009power} and its estimate $\hat\alpha$ in $q(x)$ suggest calibrated tails. 
Overall, we see that normal (resp., Cauchy) based flows fails to capture
heavy (resp., light) tails, while GGA-based
flows yield good tail approximations (lower NLL, $\hat\alpha$ closer to target) across all cases. 

\vspace{-.25cm}
\begin{table*}[h]
	\centering
	\caption{Mean and standard errors (100 trials) of 
	tail parameters $\hat\alpha$ (smaller for heavier
	tails) for various density estimators and targets.
	}
	\label{tab:de}
	
\scalebox{0.9}{
\begin{tabular}{lllll}
\toprule

Target & $\alpha$  & Cauchy ($\alpha=2$) Flow & GGA Flow & Normal ($\alpha=\infty$) Flow  \\
\midrule
Cauchy & $2$ & \bfseries $\boldsymbol{2.1}$ ($\boldsymbol{0.03}$) & \bfseries $\boldsymbol{2.1}$ ($\boldsymbol{0.07}$) & $7.7$ ($2.5$) \\
 IG & $2$ & \bfseries $\boldsymbol{1.9}$ ($\boldsymbol{0.03}$) & \bfseries $\boldsymbol{1.9}$ ($\boldsymbol{0.092}$) & $7.3$ ($1.7$) \\
 StudentT & $3$ & $2.0$ ($0.06$) & \bfseries $\boldsymbol{3.3}$ ($\boldsymbol{0.45}$) & $7.7$ ($2.3$) \\
 Chi2 & $\infty$ & $2.1$ ($0.07$) & \bfseries $\boldsymbol{5.2}$ ($\boldsymbol{1.6}$) & \bfseries $\boldsymbol{6.8}$ ($\boldsymbol{2.4}$) \\
 Normal & $\infty$ & $2.9$ ($0.6$) & \bfseries $\boldsymbol{8.2}$ ($\boldsymbol{4.0}$) & \bfseries $\boldsymbol{8.4}$ ($\boldsymbol{3.5}$) \\
\bottomrule
\end{tabular}
}

\end{table*}
\vspace{-.5cm}
\begin{table*}[h]
	\centering
	\caption{Mean and standard errors of 
	log-likelihoods $E_p \log q(X)$ for various density estimators
	and targets. While larger values imply a better overall approximation (row max bolded), log-likelihood is dominated by bulk approximation so these results show that our method (GGA Flow) does not sacrifice bulk approximation quality.
	}
	\label{tab:de_ll}
	\scalebox{0.9}{
\begin{tabular}{llllll}
\toprule

Target & $\alpha$  & Cauchy ($\alpha=2$) Flow & GGA Flow & Normal ($\alpha=\infty$) Flow  \\
\midrule
Cauchy & $2$ & \bfseries $\boldsymbol{-2.53}$ ($\boldsymbol{0.05}$) & $-3.22$ ($0.06$) & $-1.2 \times 10^3$ ($6 \times 10^3$) \\
 IG & $2$ &  $-3.55$ ($0.08$) & \bfseries $\boldsymbol{-3.26}$ ($\boldsymbol{0.05}$) & $-2.6 \times 10^4$  ($6 \times 10^3$) \\
StudentT & $3$ & \bfseries $\boldsymbol{-2.12}$ ($\boldsymbol{0.03}$) & $-2.75$ ($0.04$) & $-2.92$ ($0.47$) \\
Chi2 & $\infty$ & $-2.30$ ($0.05$) & \bfseries $\boldsymbol{-2.03}$ ($\boldsymbol{0.04}$) & $-2.24$ ($0.04$) \\
Normal & $\infty$ & $-1.53$ ($0.03$) & \bfseries $\boldsymbol{-1.41}$ ($\boldsymbol{0.02}$) & \bfseries $\boldsymbol{-1.42}$ ($\boldsymbol{0.02}$) \\
\bottomrule
\end{tabular}
}
\end{table*}

\paragraph{Variational Inference.} 
For VI, the bulk is corrected through the ELBO optimization objective $E_q \log \frac{p(X)}{q(X)} \approx \frac{1}{N} \sum_{i=1}^N \log \frac{p(x_i)}{q(x_i)},\quad x_i \sim q$.
Since the density $p$ must also be evaluated, for simplicity, experiments in \Cref{tab:vi} use closed-form marginalized densities for targets.
The overall trends also show that GGA yields consistent improvements;
the $\hat{k}$ diagnostic \citep{yao2018yes} indicates VI succeeds
($\hat{k} \leq 0.2$) when a GGA with appropriately matched tails is used and fails ($\hat{k} > 1$) when
Gaussian tails are erroneously imposed.
\vspace{-.5cm}

\begin{table*}[h]
	\centering
	\caption{Pareto $\hat{k}$ diagnostic (\cite{yao2018yes}) to assess goodness of fit for VI (mean across 100 trials, standard deviation in parenthesis)
	on targets of varying tail index (smaller $\alpha =$ heavier tails). A value $> 0.2$ is interpreted as potentially problematic so only values not exceeding it are bolded.}
	\label{tab:vi}
    \scalebox{0.9}{
\begin{tabular}{llllll}
\toprule
Target & $\alpha$ & Normal Affine & Normal Flow & GGA Affine & GGA Flow \\
 \midrule
Cauchy & $\alpha=2$ & $0.62$ ($0.26$) & $0.22$ ($0.059$) & $0.68$ ($0.038$) & \bfseries $\boldsymbol{0.091}$ ($\boldsymbol{0.04}$) \\
IG & $\alpha=2$ & $8.6$ ($1.8$) & $8.2$ ($2.3$) & $2.0$ ($0.4$) & $2.9$ ($0.71$) \\
StudentT & $\alpha=3$ & $1.2$ ($0.16$) & $1.0$ ($0.43$) & $1.5$ ($0.082$) & $1.3$ ($0.097$) \\
Chi2 & $\alpha=\infty$ & $0.57$ ($0.081$) & $0.61$ ($0.067$) & \bfseries $\boldsymbol{0.0093}$ ($\boldsymbol{0.0067}$) & \bfseries $\boldsymbol{0.089}$ ($\boldsymbol{0.044}$) \\
Normal & $\alpha=\infty$ & $0.53$ ($0.17$) & $0.21$ ($0.067$) & $0.4$ ($0.086$) & \bfseries $\boldsymbol{0.2}$ ($\boldsymbol{0.089}$) \\
\bottomrule
\end{tabular}
}
\end{table*}

\paragraph{Bayesian linear regression.} 

As a practical example of VI applied to posterior distributions, we consider the setting of one-dimensional Bayesian linear regression (BLR) with conjugate priors, defined by the likelihood $y \vert X,\beta,\sigma \sim \mathcal{N}(X\beta, \sigma^2)$ with a Gaussian prior $\beta \vert \sigma^2 \sim \mathcal{N}(0,\sigma^2)$ on the coefficients, and an inverse-Gamma prior with parameters $a_0$ and $b_0$ on the residual variance $\sigma^2$. The posterior distribution for $\beta$ conditioned on $\sigma^2$ and $X, y$ is Gaussian. However, conditional on $X, y$, $\sigma^2$ is inverse-Gamma distributed with parameters $a_0+\frac{n}{2}$ and $b_0+\frac12(y^\top y - \mu^\top \Sigma \mu))$, where $\mu = \Sigma^{-1}X^\top X \hat{\beta}$ for $\hat{\beta}$ the least-squares estimator, and $\Sigma = X^\top X + I$. Since $\sigma^2$ is positive, it is typical for PPL implementations to apply an exponential transformation. Hence, a Lipschitz normalising flow starting from a Gaussian initialization will inappropriately approximate the inverse Gamma distributed $p(\sigma^2 \vert X,y)$ with log-normal tails. On the other hand, Lipschitz flows starting from a GGA reference distribution will exhibit the correct tails. We assess this discrepancy in Figure \ref{fig:BLR} under an affine transformation on four subsampled datasets: \texttt{super} (superconductor critical temperature prediction dataset \cite{hamidieh2018data} with $n=256$ and $d = 154$); \texttt{who} (life expectancy data from the World Health Organisation in the year 2013 \cite{who-dataset} with $n=130$, $d = 18$); \texttt{air} (air quality data \cite{de2008field} with $n = 6941$, $d = 11$); and \texttt{blog} (blog feedback prediction dataset \cite{buza2013feedback} with $n = 1024$, $d = 280$). In Figure \ref{fig:BLR}(i), the GGA-based method seems to perfectly fit to the targets, while in Figure \ref{fig:BLR}(ii), the standard Gaussian approach fails to capture the tail behaviour.

\begin{figure}[h]
\centering
\qquad\quad\;(i)\hspace{0.425\textwidth} (ii)

\includegraphics[width=0.45\textwidth]{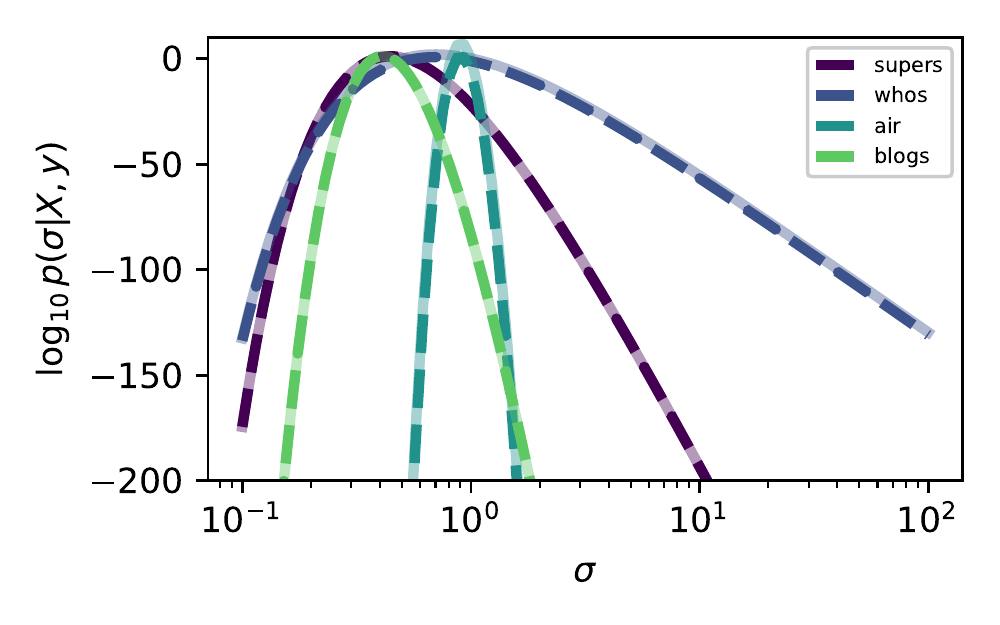}
\includegraphics[width=0.45\textwidth]{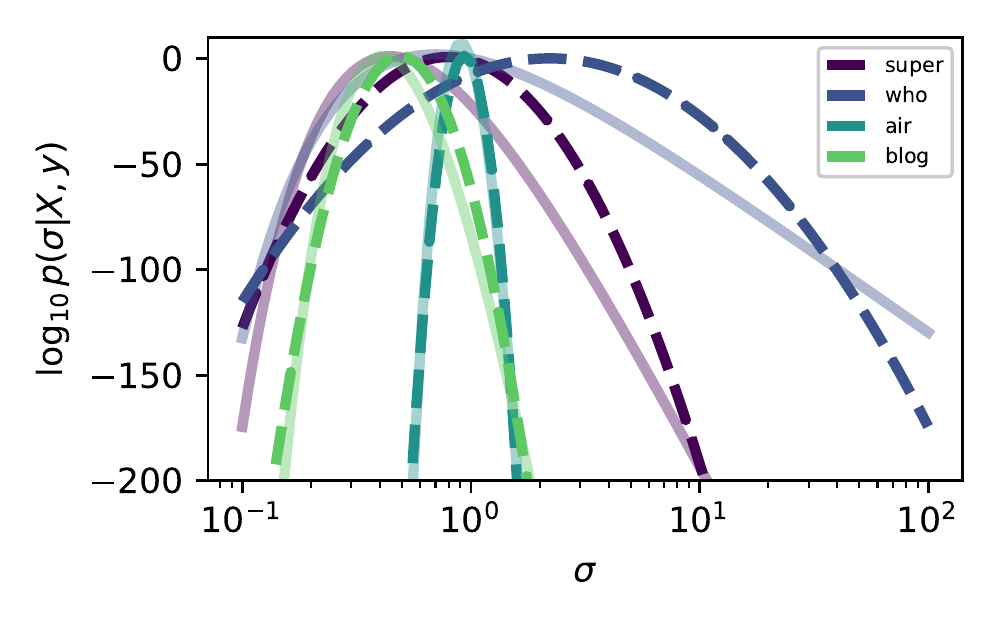}
\caption{\label{fig:BLR}Estimated densities for the posterior distribution of $\sigma^2$ in Bayesian linear regression under optimised exponential + affine transformations from (i) GGA reference, and (ii) Gaussian reference.}
\end{figure}

\paragraph{Invariant distribution of SGD.}

For inputs $X$ and labels $Y$ from a dataset $\mathcal{D}$, the least squares estimator for linear regression satisfies $\hat{\beta} = \min_\beta \tfrac12 \mathbb{E}_{X,Y\sim\mathcal{D}}(Y - X\beta)^2$. To solve for this estimator, one can apply stochastic gradient descent (SGD) sampling over independent $X_k,Y_k\sim \mathcal{D}$ to obtain the sequence of iterations
\[
\beta_{k+1} = (I - \delta X_k X_k^\top) \beta_k + \delta Y_k X_k
\]
for a step size $\delta > 0$. For large $\delta$, the iterates $\beta_k$ typically exhibit heavy-tailed fluctuations \citep{hodgkinson2021multiplicative}.
In this regard, this sequence of iterates has been used as a simple model for more general stochastic optimization dynamics \citep{gurbuzbalaban2021heavy,hodgkinson2021multiplicative}. In particular, generalization performance has been tied to the heaviness of the tails in the iterates \citep{simsekli2019tail}. Here, we use our algebra to predict the tail behaviour in a simple one-dimensional setting where $X_k \sim \mathcal{N}(0,\sigma^2)$ and $Y_k \sim \mathcal{N}(0,1)$. From classical theory \citep{buraczewski2016stochastic}, it is known that $X_k$ converges in distribution to a power law with tail exponent $\alpha > 0$ satisfying $\mathbb{E}|1 - \delta X_k^2|^\alpha = 1$. 
In \Cref{fig:SGD}, 
we plot the density of the representative for $\beta_{10^4}$ obtained using our algebra against a kernel density estimate using $10^6$ samples when $\sigma \in \{0.4,0.5,0.6\}$ and $\delta = 2$. In all cases, the density obtained from the algebra provides a surprisingly close fit. 

\begin{figure}[h]
\centering
\includegraphics[width=0.34\textwidth]{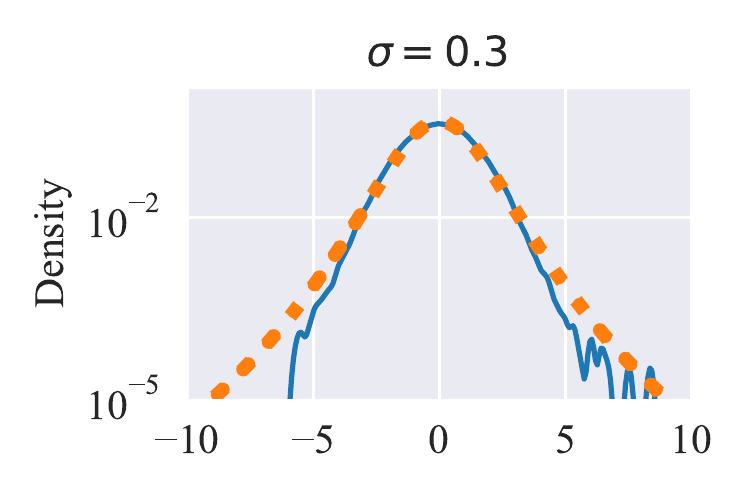}\hspace{-.4cm}
\includegraphics[width=0.34\textwidth]{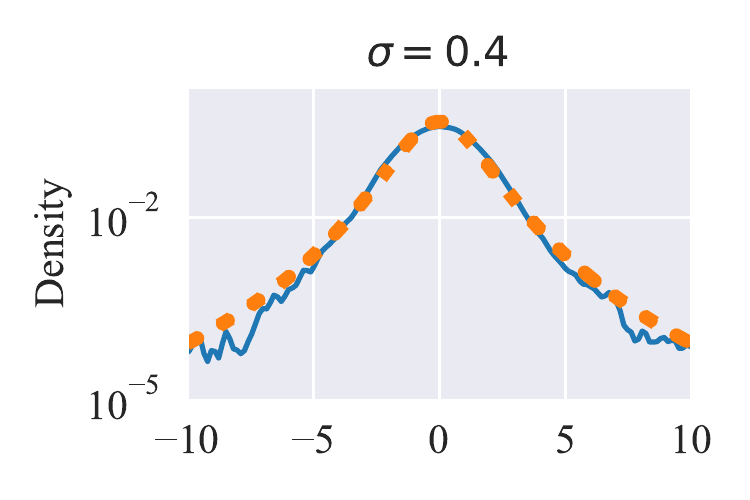}\hspace{-.4cm}
\includegraphics[width=0.34\textwidth]{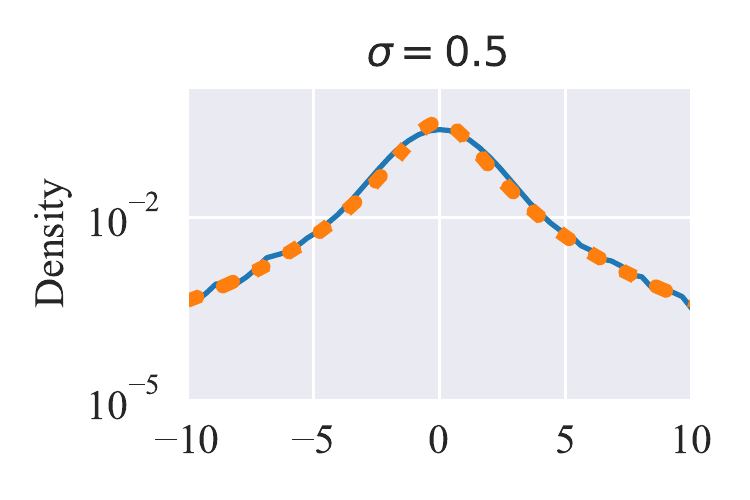}
\caption{\label{fig:SGD}Kernel density estimate of iterates of SGD (blue) vs. GGA predicted tail behaviour (orange)}
\end{figure}

\section{Related Work}

\textbf{Heavy tails and probabilistic machine learning.}
For studying heavy tails, methods based on subexponential distributions
\citep{goldie1998subexponential} and generalized Pareto distributions (GPD) (or
equivalently, regularly varying distributions \citep{tajvidi2003confidence})
have received significant attention historically. For example, \cite{mikosch} presents closure theorems for regularly varying distributions which are special cases of \Cref{prop:gga_add} and \Cref{lem:lipschitz}.
Heavy tails often have a profound impact on probabilistic machine learning methods: in particular, the observation that density ratios $\frac{p(x)}{q(x)}$ tend to be heavy tailed
has resulted in new methods for smoothing importance sampling \citep{vehtari2015pareto},
adaptively modifying divergences \citep{wang2018variational}, and
diagnosing VI through the Pareto $\hat{k}$ diagnostic \citep{yao2018yes}.
These works are complementary to our paper, and our reported results include $\hat{k}$ diagnostics for VI and $\hat\alpha$ tail index estimates based on GPD.

Our work considers heavy-tailed targets $p(x)$ which is the same setting as 
\cite{jaini2020tails,ftvi}. 
Whereas those respective works
lump the tail parameter in as another variational parameter and may be more generally applicable, the GGA may be applied before samples are drawn and leads to perfectly calibrated tails when applicable.

\textbf{Probabilistic programming.}
PPLs can be broadly characterized by the inference algorithms they support, such as:
Gibbs sampling over Bayes nets \citep{spiegelhalter1996bugs,de2017programming},
stochastic control flow \citep{goodman2012church,wingate2011lightweight},
deep stochastic VI \citep{tran2018simple,bingham2019pyro}, or
Hamiltonian Monte-Carlo \citep{carpenter2017stan,xu2020advancedhmc}. Our implementation target
\texttt{beanmachine} \citep{tehrani2020bean} is a declarative PPL selected
due to availability of a PPL compiler and support for static analysis plugins.
Similar to \cite{bingham2019pyro,siddharth2017learning}, it uses PyTorch \citep{paszke2019pytorch} for GPU tensors and automatic differentiation.
Synthesizing an approximating distribution during PPL compilation (\Cref{sec:impl}) is
also performed in the Stan language by \cite{kucukelbir2017automatic} and normalizing
flow extensions in \cite{webb2019aml}. We compare directly against these related density approximators in \Cref{sec:experiments}.

\textbf{Static analysis.}
There is a long history of formal methods and probabilistic programming in the literature \citep{kozen1979semantics,jones1989probabilistic}, with much of the research \citep{claret2013bayesian} concerned with defining formal semantics and establishing invariants \citep{wang2018pmaf} (see \cite{bernstein2019static} for a recent review).
Static analysis uses the abstract syntax tree (AST) representation of a program in order to compute invariants (e.g., the return type of a function, the number of classes implementing a trait) without executing the underlying program. 
It has traditionally been applied in the context of formalizing semantics \citep{kozen1979semantics}, and it has been used to verify probabilistic programs by ensuring termination, bounding random values values
\citep{sankaranarayanan2013static}. As dynamic analysis in a PPL is less reliable due to non-determinism, static analysis techniques for PPLs become essential. 
As a recent example, \cite{lee2019towards} proposes a static analyzer for the Pyro PPL \citep{bingham2019pyro} to verify distribution supports and avoid $-\texttt{Inf}$ log probabilities.
More relevant to our work are applications of static analysis to improve inference: \cite{nori2014r2} and \cite{cusumano2019gen} both employ static analysis to inform choice of inference method. 
However, both works do not account for heavy tails whereas the primary goal of GGA-based analysis is to ensure tails are properly modelled.

\section{Conclusion}\label{sec:conclusion}

In this work, we have proposed a novel systematic approach for conducting tail inferential static PPL analysis.
We have done this by defining a heavy-tailed algebra, and by implementing a three-parameter generalized Gamma algebra into a PPL compiler. 
Initial results are promising, showing that improved inference with simpler approximation families is possible when combined with tail metadata.
While already useful, our generalized Gamma algebra and its implementation currently have
several notable directions for improvement/extension:
\begin{itemize}[leftmargin=*]
\item 
The most significant omission to the algebra is classification of log-normal tails.
Addition may be treated using \cite{gulisashvili2016tail}, but multiplication with log-normal tails remains elusive. 
\item 
Since the algebra assumes independence, handling of dependencies between defined random variables must be conducted externally. This can be addressed using a symbolic package to decompose complex expressions into operations on independent random variables.
\item 
Scale coefficients $\sigma$ for conditional distributions may often be inexact, as exact marginalization in general is NP-hard \citep{koller2009probabilistic}. Treatment of disintegration using symbolic manipulations is a significant open problem, with some basic developments \citep{shan2017exact,cho2019disintegration}.
\item 
Compile-time static analysis is only applicable to fixed model structures. Open-universe models \citep{milch2010extending}
and PPLs to support them \citep{bingham2019pyro} are an important future research direction.
\end{itemize}
Given the recent interest in heavy-tailed aspects of machine learning more generally (e.g., see \cite{MM19_HTSR_ICML,MM20_SDM,MM20a_trends_NatComm,hodgkinson2021multiplicative,ftvi} and references therein), improving upon these directions is important future work

\paragraph{\textbf{Acknowledgments.}}
We would like to acknowledge the DOE, IARPA, NSF, and ONR as well as a J. P. Morgan Chase Faculty Research Award for providing partial support of this work.

\ifdefined\remappendix
\newpage

\appendix

\begin{center}

\Large \bf APPENDIX %

\end{center}

\section{Operations in the Generalized Gamma Algebra}\label{sec:gga_operations}

In this section, we provide explanations, references, and new results for how operations on random variables affect their GGA tails.
A summary of this, useful for referencing, appeared in  \Cref{tab:gga_operations}.

\subsection{Ordering} 
A total ordering is imposed on the equivalence classes of $\mathcal{G}$ according to the heaviness of tails. In particular, we say that $(\nu_1,\sigma_1,\rho_1) \leq (\nu_2,\sigma_2,\rho_2)$ if $(x^{\nu_1} e^{-\sigma_1 x^{\rho_1}}) / (x^{\nu_2} e^{-\sigma_2 x^{\rho_2}})$ is bounded as $x \to \infty$. As usual, we say $(\nu_1,\sigma_1,\rho_1) < (\nu_2,\sigma_2,\rho_2)$ if $(\nu_1,\sigma_1,\rho_1) \leq (\nu_2,\sigma_2,\rho_2)$ but $(\nu_1,\sigma_1,\rho_1) \not\equiv (\nu_2,\sigma_2,\rho_2)$.

\subsection{Addition} Tails of this form are  closed under addition. 
Combining subexponentiality for $\rho < 1$
\cite[Chapter X.1]{asmussen2010ruin}, 
with \cite[Thm 3.1 and Eqn. (8.3)]{asmussen2017tail}, we obtain the following Proposition \ref{prop:gga_add} for exactness of the proposed GGA addition operation. 
\begin{proposition}
\label{prop:gga_add}
Denoting the addition of random variables (additive convolution of densities) by $\oplus$,
\begin{equation}
(\nu_{1},\sigma_{1},\rho_{1})\oplus(\nu_{2},\sigma_{2},\rho_{2})\equiv \begin{cases}
\max\{(\nu_{1},\sigma_{1},\rho_{1}),(\nu_{2},\sigma_{2},\rho_{2})\} & \text{ if }\rho_{1}\neq\rho_{2}\text{ or }\rho_{1},\rho_{2}<1\\
\left(\nu_{1}+\nu_{2}+1,\min\{\sigma_{1},\sigma_{2}\},1\right) & \text{ if }\rho_{1}=\rho_{2}=1\\
(\nu_{1}+\nu_{2}+1-\frac{\rho}{2},(\sigma_{1}^{-\frac{1}{\rho-1}}+\sigma_{2}^{-\frac{1}{\rho-1}})^{1-\rho},\rho) & \text{ if }\rho=\rho_{1}=\rho_{2}>1.
\end{cases}
\end{equation}
\end{proposition}

\subsection{Powers} For all exponents $\beta > 0$, by invoking a change of variables $x \mapsto x^\beta$, it is easy to show that
$(\nu, \sigma, \rho)^\beta \equiv  \left(\frac{\nu+1}{\beta}-1, \sigma, \frac{\rho}\beta\right).$
\subsection{Reciprocals}
We \emph{define} negative powers and reciprocals equivalently to positive powers in the case $\beta < 0$. This equivalence cannot be proven to hold in general since we cannot determine tail asymptotics of the reciprocal without knowledge of its behaviour around zero. Therefore, we implicitly assume that the behaviour around zero mimics the tail behaviour, that is, \Cref{eq:GenGammaTails} holds as $x \to 0^+$. Note that this can only hold provided $(\nu + 1)/\rho > 0$ and $\rho \neq 0$. To account for all other cases, including $\mathcal{R}_\nu$, we assume that the density of $X$ approaches some nonzero value near zero. In this case, Lemma \ref{lem:RecipCauchy} defines the reciprocal to be $\mathcal{R}_{2}$. 
\begin{lemma}
\label{lem:RecipCauchy}
Assume that a random variable $X$ has a density $p$ that is continuous at zero and $p(0) > 0$. Then $X^{-1} \equiv \mathcal{R}_2$. 
\end{lemma}
\begin{proof}
From a change of variables, the density $q$ of $X^{-1}$ is given by $q(x) = |x|^{-2}p(x^{-1})$. By assumption, as $|x|\to \infty$, $q(x) \sim p(0) |x|^{-2}$. Therefore, $X^{-1} \equiv \mathcal{R}_2$.
\end{proof}

\subsection{Multiplication}
For any $c \in \mathbb{R} \backslash \{0\}$, it can be readily seen from a change of variables $x \mapsto c x$ that $c (\nu, \sigma, \rho) = (\nu, \sigma / |c|^\rho, \rho)$. However, the case of multiplication convolution is not as straightforward. While additive convolutions of generalized Gamma random variables are relatively well-explored, to our knowledge, multiplicative convolution has not been examined at this level of generality. It turns out that the class $\mathcal{G}$ is also closed under multiplication (assuming independence of random variables), as we show in the following result. The proof requires some preliminary background on Mellin transforms and the Fox H function, which we cover in Appendix \ref{sec:Mellin}.

\begin{proposition}
\label{prop:Mult}
Denoting the multiplication of independent random variables (multiplicative convolution) by $\otimes$,
\[
(\nu_{1},\sigma_{1},\rho_{1})\otimes(\nu_{2},\sigma_{2},\rho_{2})
\equiv\begin{cases}
\left(\frac{1}{\mu}\left(\frac{\nu_{1}}{|\rho_{1}|}+\frac{\nu_{2}}{|\rho_{2}|}+\frac{1}{2}\right),\sigma,-\frac{1}{\mu}\right) & \text{ if }\rho_{1},\rho_{2}<0\\
\left(\frac{1}{\mu}\left(\frac{\nu_{1}}{\rho_{1}}+\frac{\nu_{2}}{\rho_{2}}-\frac{1}{2}\right),\sigma,\frac{1}{\mu}\right) & \text{ if }\rho_{1},\rho_{2}>0\\
\mathcal{R}_{|\nu_1|} & \mbox{ if }\rho_{1}\leq0,\rho_{2}>0 \\
\mathcal{R}_{\min\{|\nu_1|,|\nu_2|\}} & \mbox{ if }\rho_{1}=0,\rho_{2}=0
\end{cases}
\]
where $\mu=\frac{1}{|\rho_{1}|}+\frac{1}{|\rho_{2}|}=\frac{|\rho_{1}|+|\rho_{2}|}{|\rho_{1}\rho_{2}|}$ and $\sigma=\mu(\sigma_{1}|\rho_{1}|)^{\frac{1}{\mu|\rho_{1}|}}(\sigma_{2}|\rho_{2}|)^{\frac{1}{\mu|\rho_{2}|}}$. 
\end{proposition}
\begin{proof}
The $\rho_1 \leq 0, \rho_2 > 0$ and $\rho_1 = \rho_2 = 0$ cases follow from Breiman's lemma \cite[Lemma B.5.1]{buraczewski2016stochastic}. Our argument proceeds similar to \cite{asmussen2017tail}. Assume that $\rho_1,\rho_2 > 0$ and let $0 < \epsilon < 1$ be such that $0 < a_- < a_+ < 1$, where
\[
a_+ = \frac{(1+\epsilon)\rho_2}{\rho_1+\rho_2},\qquad a_{-} = 1 - \frac{(1+\epsilon)\rho_1}{\rho_1+\rho_2}.
\]
Then for $\rho = \frac{\rho_1\rho_2}{\rho_1 + \rho_2}$, if $X \equiv (\nu_1,\sigma_1,\rho_1)$ and $Y \equiv (\nu_2,\sigma_2,\rho_2)$, then
\begin{align*}
\mathbb{P}(XY>x,X\notin[x^{a_{-}},x^{a_{+}}])  & \leq\mathbb{P}(X>x^{a_{+}})+\mathbb{P}(Y>x^{1-a_{-}}) \\
& \sim c_{1}x^{\nu_{1}a_{+}}e^{-\sigma_{1}x^{\rho_{1}a_{+}}}+c_{2}x^{\nu_{2}(1-a_{-})}e^{-\sigma_{2}x^{\rho_{2}(1-a_{-})}}\\
& \leq\left(c_{1}x^{\nu_{1}a_{+}}+c_{2}x^{\nu_{2}(1-a_{-})}\right)e^{-\min\{\sigma_{1},\sigma_{2}\}x^{(1+\epsilon)\rho}} = o(x^\nu e^{-\sigma x^\rho}),
\end{align*}
for any $\nu,\sigma > 0$. Hence, it will suffice to show the claimed tail asymptotics for the generalized Gamma distribution. In this case, since $a_- > 0$ and $a_+ < 1$, the tail of the distribution for the product of $X,Y$ depends only on the tail of the distributions for $X$ and $Y$. 

Therefore, assume without loss of generality that $p_X(x) = c_X x^{\nu_1} e^{-\sigma_1 x^{\rho_1}}$ and $p_Y(x) = c_Y x^{\nu_2} e^{-\sigma_2 x^{\rho_2}}$. Then
\[
\mathcal{M}_s[p_{XY}] = c_X c_Y \frac{\sigma_{1}^{-\nu_{1}/\rho_{1}}}{\rho_{1}}\frac{\sigma_{2}^{-\nu_{2}/\rho_{2}}}{\rho_{2}}\left(\sigma_{1}^{1/\rho_{1}}\sigma_{2}^{1/\rho_{2}}\right)^{-s}\Gamma\left(\frac{\nu_{1}}{\rho_{1}}+\frac{s}{\rho_{1}}\right)\Gamma\left(\frac{\nu_{2}}{\rho_{2}}+\frac{s}{\rho_{2}}\right).
\]
Consequently,
\[
p_{XY}(z) = c_X c_Y \frac{\sigma_{1}^{-\nu_{1}/\rho_{1}}}{\rho_{1}}\frac{\sigma_{2}^{-\nu_{2}/\rho_{2}}}{\rho_{2}}H_{0,2}^{2,0}\left[\sigma_{1}^{1/\rho_{1}}\sigma_{2}^{1/\rho_{2}}z\left|\substack{-\\
(\frac{\nu_{1}}{\rho_{1}},\frac{1}{\rho_{1}}),(\frac{\nu_{2}}{\rho_{2}},\frac{1}{\rho_{2}})
}
\right.\right]
\]
Computing the corresponding $\beta,\delta,\mu$ for the asymptotic expansion, we find that
\[
\mu=\frac{1}{\rho_{1}}+\frac{1}{\rho_{2}},\qquad\delta=\frac{\nu_{1}}{\rho_{1}}+\frac{\nu_{2}}{\rho_{2}}-1,\qquad\beta=\rho_{1}^{-1/\rho_{1}}\rho_{2}^{-1/\rho_{2}}.
\]
Consequently, for some $c > 0$,
\[
p_{XY}(z) \sim c z^{\frac{1}{\mu}(\frac{1}{2}+\delta)}\exp\left(-\mu\beta^{-\frac{1}{\mu}}(\sigma_{1}^{1/\rho_{1}}\sigma_{2}^{1/\rho_{2}})^{\frac{1}{\mu}}z^{\frac{1}{\mu}}\right),
\]
which completes the $\rho_1,\rho_2 > 0$ case. The final case follows by composing the multiplication and reciprocal operations. Note that
\begin{align*}
(\nu_{1},\sigma_{1},-\rho_{1})^{-1}\otimes(\nu_{2},\sigma_{2},-\rho_{2})^{-1}	&\equiv\left(-\nu_{1}-2,\sigma_{1},\rho_{1}\right)\otimes\left(-\nu_{2}-2,\sigma_{2},\rho_{2}\right)\\
&\equiv\left(\frac{1}{\mu}\left(\frac{-\nu_{1}-2}{\rho_{1}}+\frac{-\nu_{2}-2}{\rho_{2}}-\frac{1}{2}\right),\sigma,\frac{1}{\mu}\right)\\
&\equiv\left(\frac{1}{\mu}\left(\frac{-\nu_{1}}{\rho_{1}}+\frac{-\nu_{2}}{\rho_{2}}-2\mu-\frac{1}{2}\right),\sigma,\frac{1}{\mu}\right)\\
&\equiv\left(\frac{1}{\mu}\left(\frac{-\nu_{1}}{\rho_{1}}+\frac{-\nu_{2}}{\rho_{2}}-\frac{1}{2}\right)-2,\sigma,\frac{1}{\mu}\right),
\end{align*}
and therefore
\[
(\nu_{1},\sigma_{1},-\rho_{1})\otimes(\nu_{2},\sigma_{2},-\rho_{2})\equiv\left(\frac{1}{\mu}\left(\frac{\nu_{1}}{\rho_{1}}+\frac{\nu_{2}}{\rho_{2}}+\frac{1}{2}\right),\sigma,-\frac{1}{\mu}\right).
\]
\end{proof}

\subsection{Product of Densities}

We can also consider a product of densities operation acting on two random variables $X,Y$, denoted $X \& Y$, by
$p_{X \& Y}(x) = c p_X(x) p_Y(x)$,
where $c > 0$ is an appropriate normalizing constant and $p_X,p_Y,p_{X\& Y}$ are the densities of $X$, $Y$, and $X \& Y$, respectively. In terms of the equivalence classes:
\[
(\nu_1,\sigma_1,\rho_1)\&(\nu_2,\sigma_2,\rho_2) 
\equiv \begin{cases}
(\nu_{1}+\nu_{2},\sigma_{1},\rho_{1}) & \text{ if }\rho_{1}<\rho_{2}\\
(\nu_{1}+\nu_{2},\sigma_{1}+\sigma_{2},\rho) & \text{ if }\rho=\rho_{1}=\rho_{2}\\
(\nu_{1}+\nu_{2},\sigma_{2},\rho_{2}) & \text{ otherwise,}
\end{cases}
\]
which follows directly by taking the product of the generalized Gamma tails in \cref{eq:GenGammaTails}. 
Note that this particular operation does not require either $p_X$ or $p_Y$ to be normalized --- only the tail behaviour is needed. We may also use this to work out the tail behaviour of a posterior density, provided the tail behaviour of the likelihood in the parameters is known.

\subsection{Exponential and Logarithm} Tails of the generalized Gamma form are not closed under exponentiation or logarithms. Indeed, if both $X$ and $\exp X$ have generalized Gamma tails, then $X$ is exponentially distributed (and $\exp X$ has power law tails). As a workaround, we can consider an upper bound on the tail by projecting onto the nearest possible exponentially distributed / power law tail. If $\rho > 1$, then a change of variables shows the density of $\exp X$ satisfies
\[
p_{\exp X}(x) \sim \frac{c}{x}(\log x)^{\nu}\exp\left(-\sigma(\log x)^{\rho}\right)\leq\frac{\tilde{c}}{x}\exp\left(-\sigma(\log x)\right)=cx^{-\sigma-1},\,\mbox{ as }x \to \infty.
\]
The inverse of this operation sends $\mathcal{R}_{\sigma+1}$ to $(0,\sigma,1)$. With this in mind, we define the exponential and logarithmic operations according to the following: $\exp (\nu, \sigma, \rho) \equiv \mathcal{R}_{\sigma+1}$ if $\rho \geq 1$, otherwise $\mathcal{R}_1$; $\log (\nu, \sigma, \rho) \equiv (0, |\nu|-1, 1)$ if $\nu < -1$ and $\rho \leq 0$, otherwise $\mathcal{L}$.

\subsection{Lipschitz Functions}

There are many multivariate functions that cannot be readily represented in terms of the operations covered thus far. For these, it is important to specify the tail behaviour of pushforward measures under Lipschitz-continuous functions. Fortunately, this is covered by \Cref{lem:lipschitz} below, presented in \cite[Proposition 1.3]{ledoux2001concentration}. 
\begin{theorem}\label{lem:lipschitz}
For any Lipschitz continuous function $f:\mathbb{R}^d \to \mathbb{R}$ satisfying $\|f(x)-f(y)\|\leq L\|x -y\|$ for $x,y \in \mathbb{R}^d$, there is $f(X_1,\dots,X_d) \equiv L \max\{X_1,\dots,X_d\}.$ More generally, for any H\"{o}lder continuous function $f:\mathbb{R}^d\to\mathbb{R}$ satisfying $\|f(x)-f(y)\| \leq L\|x-y\|^\alpha$ for $x,y\in\mathbb{R}^d$, there is $f(X_1,\dots,X_d) \equiv L \max\{X_1^\alpha,\dots,X_d^\alpha\}$. 
\end{theorem}

\subsection{Power Law Approximation} 

There are many cases where power laws arise not from a single operation of random variables, but cumulatively, through many successive operations. In these cases, $\rho$ becomes small while $\sigma$ becomes large, such that $\sigma = \mathcal{O}(\rho^{-1})$. To see how this regime induces a power law, note that as $x \to \infty$, 
\[
p_{|X|}(x) \sim c x^\nu e^{-\sigma x^\rho} = \tilde{c} x^\nu e^{-\sigma(x^\rho - 1)} = \tilde{c} x^\nu e^{-\sigma\rho\frac{x^\rho - 1}{\rho}} \approx \tilde{c} x^\nu e^{-\sigma\rho \log x} = \tilde{c} x^{\nu-\sigma \rho},
\]
where we have used the approximation $\log x = \rho^{-2}(x^\rho - 1) + \mathcal{O}(\rho^2)$. Consequently, we can represent tails of this form by the Student $t$ distribution with $|\nu-\sigma\rho|-1$ degrees of freedom. In practice, we find this approximation tends to \emph{overestimate} the heaviness of the tail. 

Alternatively, the generalized Gamma density (\ref{eq:GenGammaDensity}) satisfies $\mathbb{E}X^r = \sigma^{-r/\rho} \Gamma(\frac{\nu+1+r}{\rho})/\Gamma(\frac{\nu+1}{\rho})$ for $r > 0$. Let $\alpha > 0$ be such that $\mathbb{E}X^\alpha = 2$. By Markov's inequality, the tail of $X$ satisfies $\mathbb{P}(X>x)\leq 2 x^{-\alpha}$. Therefore, we can represent tails of this form by the Student $t$ distribution with $\alpha+1$ degrees of freedom (generate $X \sim \text{StudentT}(\alpha)$). In practice, we find this approximation to be more accurate, and is hence used as our power law candidate distribution in Section~\ref{ssec:repr_dist}.

\subsection{Posterior Distributions} Suppose that a random variable $X$ is dependent on a parameter $\theta$ and a latent random element $Z$ through a function $f$ by $X = f(Z;\theta)$. Letting $\pi$ denote a prior on $\theta$, since $p(\theta\vert x) \propto p_X(x \vert \theta) \pi(\theta)$, it will suffice to find the tail of $p_X(x\vert \theta)$ in $\theta$, as we can incorporate the tail of $\pi$ with the \& operation. Assuming that $f$ is invertible with respect to both $Z$ and $\theta$ with respective inverses $f^{-1}(x;\theta)$ and $\Theta(x;z)$, a change of variables shows that $$p_X(x\vert \theta) = p_Z(f^{-1}(x;\theta))\left|\frac{\partial}{\partial x} f^{-1}(x;\theta)\right|.$$ Note that $z = f^{-1}(x;\Theta(x;z))$ and so $\Theta^{-1}(\theta;x) = f^{-1}(x;\theta)$, where $\Theta^{-1}(x;\theta)$ is the inverse of $z \mapsto \Theta(x;z)$ at $\theta$. Therefore, the density of $\Theta(x;Z)$ is $$p_{\Theta}(\theta,x) = p_Z(f^{-1}(x;\theta))\left|\frac{\partial}{\partial\theta} f^{-1}(x;\theta)\right|.$$ Consequently,
$$
p_X(x\vert \theta) = p_\Theta(\theta,x) R(x,\theta),
$$
where $R(x,\theta) = |\frac{\partial}{\partial x} f^{-1}(x;\theta)|/|\frac{\partial}{\partial\theta} f^{-1}(x;\theta)|$. Since the inverse of a composition of operations is a composition of inverses, the tail of $p_\Theta$ is relatively straightforward to determine by tracing back through the computation graph and sequentially applying inverse operations, i.e., $\oplus$ (addition) becomes $\ominus$ (subtraction), etc. For example, if $X = \mu + Z$, then $f(z,\mu) = \mu + z$, $f^{-1}(x;\mu) = x - \mu$, and $R(x,\mu) = 1$. Therefore, $\mu \vert X = x \equiv (x - z) \,\&\, \pi$. Similarly, if $X = \sigma Z$, then $f(z,\sigma) = \sigma z$, $f^{-1}(x,\sigma) = x/\sigma$, and $R(x,\sigma) = \sigma^{-1}/(x\sigma^{-2}) \equiv \sigma$. Therefore, $\mu \vert X = x \equiv (x/Z)\,\&\,(1,1,0)\,\&\,\pi$. If $X = Z/\sigma$, then $f(Z,\sigma) = z/\sigma$, $f^{-1}(x,\sigma) = \sigma x$, and $R(x,\sigma) = \sigma / x \equiv \sigma$. 

\section{List of Univariate Distributions}\label{sec:univariate_classes}

To demonstrate the scope of our algebra and facilitate implementation in a general PPL, Table \ref{tab:dist_list} lists many families of one-dimensional densities and their corresponding tail class. 

\bgroup
\def\arraystretch{2}
\begin{longtable}{|cccc|}
\caption{List of univariate distributions}\label{tab:dist_list}\\
\hline 
Name & Support & Density $p(x)$ & Class\tabularnewline
\hline 
\hline 
\small Benktander Type II & $(0,\infty)$ & $e^{\frac{a}{b}(1-x^{b})}x^{b-2}(ax^{b}-b+1)$ & $(2b-2,\frac{a}{b},b)$\tabularnewline
 
\small Beta prime  & $(0,\infty)$ & $\frac{\Gamma(\alpha+\beta)}{\Gamma(\alpha)\Gamma(\beta)}x^{\alpha-1}(1+x)^{-\alpha-\beta}$ & $\mathcal{R}_{\beta+1}$\tabularnewline
 
\small Burr  & $(0,\infty)$ & $ckx^{c-1} (1+x^{c})^{-k-1}$ & $\mathcal{R}_{ck+1}$\tabularnewline
 
\small Cauchy  & $(-\infty,\infty)$ & $(\pi\gamma)^{-1}\left[1+\left(\frac{x-x_{0}}{\gamma}\right)^{2}\right]^{-1}$ & $\mathcal{R}_{2}$\tabularnewline
 
\small Chi  & $(0,\infty)$ & $\frac{1}{2^{k/2-1}\Gamma(k/2)}x^{k-1}e^{-x^{2}/2}$ & $(k-1,\frac{1}{2},2)$\tabularnewline
 
\small Chi-squared  & $(0,\infty)$ & $\frac{1}{2^{k/2}\Gamma(k/2)}x^{\frac{k}{2}-1}e^{-x/2}$ & $(\frac{k}{2}-1,\frac{1}{2},1)$\tabularnewline
 
\small Dagum  & $(0,\infty)$ & $\frac{ap}{x}\left(\frac{x}{b}\right)^{ap}\left(\left(\frac{x}{b}\right)^{a}+1\right)^{-p-1}$ & $\mathcal{R}_{a+1}$\tabularnewline
 
\small Davis  & $(0,\infty)$ & $\propto(x-\mu)^{-1-n}/\left(e^{\frac{b}{x-\mu}}-1\right)$ & $(-1-n,b,-1)$\tabularnewline
 
\small Exponential  & $(0,\infty)$ & $\lambda e^{-\lambda x}$ & $(0,\lambda,1)$\tabularnewline
 
\small $F$  & $(0,\infty)$ & $\propto x^{d_{1}/2-1}(d_{1}x+d_{2})^{-(d_{1}+d_{2})/2}$ & $\mathcal{R}_{d_{2}/2+1}$\tabularnewline
 
\small Fisher $z$ & $(-\infty,\infty)$ & $\propto\frac{e^{d_{1}x}}{(d_{1}e^{2x}+d_{2})^{(d_{1}+d_{2})/2}}$ & $(0,d_{2},1)$\tabularnewline
 
\small Frechet  & $(0,\infty)$ & $\frac{\alpha}{\lambda}\left(\frac{x-m}{\lambda}\right)^{-1-\alpha}e^{-\left(\frac{x-m}{\lambda}\right)^{-\alpha}}$ & $(-1-\alpha,\lambda^{\alpha},-\alpha)$\tabularnewline
 
\small Gamma  & $(0,\infty)$ & $\frac{\beta^{\alpha}}{\Gamma(\alpha)}x^{\alpha-1}e^{-\beta x}$ & $(\alpha-1,\beta,1)$\tabularnewline
 
\small Gamma/Gompertz  & $(0,\infty)$ & $bse^{bx}\beta^{s}/(\beta-1+e^{bx})^{s+1}$ & $(0,bs,1)$\tabularnewline
 
\small Generalized hyperbolic  & $(-\infty,\infty)$ & $\propto e^{\beta(x-\mu)}\frac{K_{\lambda-1/2}(\alpha\sqrt{\delta^{2}+(x-\mu)^{2}})}{(\delta^{2}+(x-\mu)^{2})^{1/4-\lambda/2}}$ & $(\lambda-1,\alpha-\beta,1)$\tabularnewline
 
\small Generalized normal  & $(-\infty,\infty)$ & $\frac{\beta}{2\alpha\Gamma(1/\beta)}\exp\left(-\left(\frac{|x-\mu|}{\alpha}\right)^{\beta}\right)$ & $(0,\alpha^{-\beta},\beta)$\tabularnewline
 
\small Geometric stable  & $(-\infty,\infty)$ & no closed form & $\mathcal{R}_{\alpha+1}$\tabularnewline
 
\small Gompertz  & $(0,\infty)$ & $\sigma \eta \exp(\eta + \sigma x - \eta e^{\sigma x})$ & $\mathcal{L}$\tabularnewline

\small Gumbel  & $(0,\infty)$ & $\beta^{-1} e^{-(\beta^{-1}(x-\mu)+e^{-\beta^{-1}(x-\mu)})}$ & $(0,\frac{1}{\beta},1)$\tabularnewline

\small Gumbel Type II  & $(0,\infty)$ & $\alpha\beta x^{-\alpha-1}e^{-\beta x^{-\alpha}}$ & $(-\alpha-1,\beta,-\alpha)$\tabularnewline
 
\small Holtsmark  & $(-\infty,\infty)$ & no closed form & $\mathcal{R}_{5/2}$\tabularnewline
 
\small Hyperbolic secant  & $(-\infty,\infty)$ & $\frac{1}{2}\text{sech}\left(\frac{\pi x}{2}\right)$ & $(0,\frac{\pi}{2},1)$\tabularnewline
 
\small Inverse chi-squared  & $(0,\infty)$ & $\frac{2^{-k/2}}{\Gamma(k/2)}x^{-k/2-1}e^{-1/(2x)}$ & $(-\frac{k}{2}-1,\frac{1}{2},-1)$\tabularnewline
 
\small Inverse gamma  & $(0,\infty)$ & $\frac{\beta^{\alpha}}{\Gamma(\alpha)}x^{-\alpha-1}e^{-\beta/x}$ & $(-\alpha-1,\beta,-1)$\tabularnewline
 
\small Levy  & $(0,\infty)$ & $\sqrt{\frac{c}{2\pi}}(x-\mu)^{-3/2}e^{-\frac{c}{2(x-\mu)}}$ & $(-\frac{3}{2},\frac{c}{2},-1)$\tabularnewline
 
\small Laplace  & $(-\infty,\infty)$ & $\frac{1}{2\lambda}\exp\left(-\frac{|x-\mu|}{\lambda}\right)$ & $(0,\frac{1}{\lambda},1)$\tabularnewline
 
\small Logistic  & $(-\infty,\infty)$ & $\frac{e^{-(x-\mu)/\lambda}}{\lambda(1+e^{-(x-\mu)/\lambda})^{2}}$ & $(0,\frac{1}{\lambda},1)$\tabularnewline
 
\small Log-Cauchy  &
$(0,\infty)$ & $\frac{\sigma}{x\pi}((\log x - \mu)^2 + \sigma^2)^{-1}$ & $\mathcal{R}_1$\tabularnewline

\small Log-Laplace  & $(0,\infty)$ & $\frac{1}{2\lambda x}\exp\left(-\frac{\left|\log x-\mu\right|}{\lambda}\right)$ & $\mathcal{R}_{1/\lambda+1}$\tabularnewline
 
\small Log-logistic  & $(0,\infty)$ & $\frac{\beta}{\alpha}\left(\frac{x}{\alpha}\right)^{\beta-1}\left(1+\left(\frac{x}{\alpha}\right)^{\beta}\right)^{-2}$ & $\mathcal{R}_{\beta+1}$\tabularnewline

\small Log-$t$  &
$(0,\infty)$ &
$\propto x^{-1} (1 + \frac1{\nu}(\log x - \mu)^2)^{-\frac{\nu+1}{2}}$ & $\mathcal{R}_1$\tabularnewline
 
\small Lomax  & $(0,\infty)$ & $\frac{\alpha}{\lambda}\left(1+\frac{x}{\lambda}\right)^{-\alpha-1}$ & $\mathcal{R}_{\alpha+1}$\tabularnewline
 
\small Maxwell-Boltzmann  & $(0,\infty)$ & $\sqrt{\frac{2}{\pi}}\frac{x^{2}e^{-x^{2}/(2\sigma^{2})}}{\sigma^{3}}$ & $(2,\frac{1}{2\sigma^{2}},2)$\tabularnewline
 
\small normal  & $(-\infty,\infty)$ & $\frac{1}{\sigma\sqrt{2\pi}}e^{-\frac{1}{2}(\frac{x-\mu}{\sigma})^{2}}$ & $(0,\frac{1}{2\sigma^{2}},2$)\tabularnewline
 
\small Pareto  & $(x_{0},\infty)$ & $\alpha x_{0}^{\alpha}x^{-\alpha-1}$ & $\mathcal{R}_{\alpha+1}$\tabularnewline
 
\small Rayleigh  & $(0,\infty)$ & $\frac{x}{\sigma^{2}}e^{-x^{2}/(2\sigma^{2})}$ & $(1,\frac{1}{2\sigma^{2}},2)$\tabularnewline
 
\small Rice  & $(0,\infty)$ & $\frac{x}{\sigma^{2}}\exp\left(-\frac{(x^{2}+\nu^{2})}{2\sigma^{2}}\right)I_{0}\left(\frac{x\nu}{\sigma^{2}}\right)$ & $(\frac{1}{2},\frac{1}{2\sigma^{2}},2)$\tabularnewline
 
\small Skew normal  & $(-\infty,\infty)$ & no closed form & $(0,\frac{1}{2\sigma^{2}},2)$\tabularnewline
 
\small Slash  & $(-\infty,\infty)$ & $\frac{1-e^{-\frac{1}{2}x^{2}}}{\sqrt{2\pi}x^{2}}$ & $(-2,\frac{1}{2},2)$\tabularnewline
 
\small Stable  & $(-\infty,\infty)$ & no closed form & $\mathcal{R}_{\alpha+1}$\tabularnewline
 
\small Student's $t$- & $(-\infty,\infty)$ & $\frac{\Gamma(\frac{\nu+1}{2})}{\sqrt{\nu\pi}\Gamma(\frac{\nu}{2})}\left(1+\frac{x^{2}}{\nu}\right)^{-\frac{\nu+1}{2}}$ & $\mathcal{R}_{\nu+1}$\tabularnewline
 
\small Tracy-Widom  & $(-\infty,\infty)$ & no closed form & $(-\frac{3\beta}{4}-1,\frac{2\beta}{3},\frac{3}{2})$\tabularnewline
 
\small Voigt  & $(-\infty,\infty)$ & no closed form & $\mathcal{R}_{2}$\tabularnewline
 
\small Weibull  & $(0,\infty)$ & $\frac{\rho}{\lambda}\left(\frac{x}{\lambda}\right)^{\rho-1}e^{-(x/\lambda)^{\rho}}$ & $(\rho-1,\lambda^{-\rho},\rho)$\tabularnewline
\hline 
\end{longtable}
\egroup

The following densities are not supported by our algebra: Benini distribution; Benktander Type I distribution; Johnson's $S_U$-distribution; and the log-normal distribution. All of these densities exhibit log-normal tails.

\section{Additional Details for Experiments}\label{sec:experiment_details}

The targets in \Cref{tab:de} and \Cref{tab:vi} are analyzed using
the GGA in \Cref{sec:addtl_eg}. 
Note that Inverse Gamma (``IG'') corresponds
to the inverse exponential. We selected closed form targets so that
the Pareto tail index $\alpha$ is known analytically and the quality
of theoretical predictions as well as empirical results can be rigorously evaluated. 
All experiments are repeated on i7-8700K CPU and GTX 1080 GPU hardware for $100$ trials. $1000$ samples
from the model (as well as the approximation in VI) were used to compute
each gradient estimate. Losses were trained until convergence, which 
all occurred in under $10^4$ iterations at a $0.05$ learning rate and
the Adam \citep{kingma2014adam} optimizer.

\section{Mellin Transforms}
\label{sec:Mellin}

Recall that the Mellin transform of a function $f$ on $(0,\infty)$ is given by
\[
\mathcal{M}_s[f] = \int_0^\infty x^{s-1} f(x) \dd x.
\]
Letting $p_{XY}$ denote the density of the product of independent random variables $X,Y$ with respective densities $p_X$ and $p_Y$, $\mathcal{M}_s[p_{XY}] = \mathcal{M}_s[p_X] \mathcal{M}_s[p_Y]$. There is
\[
\mathcal{M}_s[c x^\nu e^{-\sigma x^\rho}] = \frac{c \sigma^{-\nu/\rho}}{\rho} \sigma^{-s/\rho} \Gamma\left(\frac{\nu}{\rho} + \frac{s}{\rho}\right).
\]
To facilitate the proof of Proposition \ref{prop:Mult}, we define the Fox $H$-function
\[
H_{p,q}^{m,n}\left[z\left|\substack{(a_{1},A_{1}),\dots,(a_{p},A_{p})\\
(b_{1},B_{1}),\dots,(b_{q},B_{q})
}
\right.\right]
\]
as the inverse Mellin transform of
\[
\Theta(s) = z^{-s} \frac{\prod_{j=1}^m \Gamma(b_j + B_j s) \cdots \prod_{j=1}^n \Gamma(1-a_j - A_j s)}{\prod_{j=m+1}^q \Gamma(1-b_j - B_j s) \prod_{j=n+1}^p \Gamma(a_j + A_j s)}.
\]
An important property of the Fox $H$-function is its asymptotic behaviour as $z \to \infty$. From \cite[Theorem 1.3]{mathai2009h}, we have
\[
H_{p,q}^{q,0}\left[z\left|\substack{(a_{1},A_{1}),\dots,(a_{p},A_{p})\\
(b_{1},B_{1}),\dots,(b_{q},B_{q})}\right.\right] \sim c x^{(\delta+\frac12)/\mu} \exp(-\mu \beta^{-1/\mu} x^{1/\mu}),\qquad \mbox{as }x\to\infty,
\]
for some constant $c > 0$, where $\beta = \prod_{j=1}^p (A_j)^{-A_j} \prod_{j=1}^q B_j^{B_j}$, $\mu = \sum_{j=1}^q B_j - \sum_{j=1}^p A_j$, and $\delta = \sum_{j=1}^q b_j - \sum_{j=1}^p a_j + \frac{p-q}{2}$.

\fi

\section{Proof of Lemma \ref{lem:Products}}
\label{sxn:proof_of_a_lemma}

The proof relies on the following integral definition \cite[pg. 183]{watson1995treatise} and asymptotic relation as $z \to \infty$ \cite[pg. 202]{watson1995treatise} of the modified Bessel function $K_\nu(z)$ for $z > 0$ and $\nu \geq 0$,
\begin{equation}
\label{eq:BesselInt}
K_\nu(z) = \frac12 \left(\frac{z}{2}\right)^\nu \int_0^\infty u^{-\nu-1}\exp\left(-u - \frac{z^2}{4u}\right) \dd u \sim \sqrt{\frac{\pi}{2z}}e^{-z}. 
\end{equation}
We also make use of the known density for the product of two independent continuous random variables: if $X$ and $Y$ have densities $p_X$ and $p_Y$ respectively, then $Z = XY$ has density
\[
p_Z(z) = \int_\mathbb{R} p_X(x) p_Y(z/x) |x|^{-1} \dd x.
\]
\begin{itemize}
\item \textbf{Density of $X_1 X_2$:} Recalling that the density of $X \sim \mathrm{Exp}(\lambda)$ is $p_X(x) = \lambda e^{-\lambda x}$ for $x \geq 0$, for $Z = XY$ where $X \sim \mathrm{Exp}(\lambda_1)$ and $Y \sim \mathrm{Exp}(\lambda_2)$ are independent,
\[
p_Z(z) = \int_{0}^{\infty}x^{-1}\lambda_{1}e^{-\lambda_{1}x}\lambda_{2}e^{-\lambda_{2}z/x}\dd x = \lambda_{1}\lambda_{2}\int_{0}^{\infty}x^{-1}e^{-\lambda_{1}x-\lambda_{2}z/x}\dd x.
\]
Since $2K_{0}(2\sqrt{z})=\int_{0}^{\infty}u^{-1}\exp(-u-\frac{z}{u})\dd u$, let $u = \lambda_1 v$, so that $\dd u = \lambda_1 \dd v$,
\[
2K_{0}(2\sqrt{\lambda_{1}\lambda_{2}z})=\int_{0}^{\infty}u^{-1}\exp\left(-\lambda_{1}v-\lambda_{2}\frac{z}{v}\right)\dd v.
\]
Therefore, letting $\lambda = \sqrt{\lambda_1\lambda_2}$, 
\[
p_Z(z) = 2\lambda^2 K_0(2\lambda\sqrt{z}) \sim \sqrt{\pi} \lambda^{3/2} z^{-1/4} e^{-2\lambda z^{1/2}}.
\]
\item \textbf{Density of $Z_1 Z_2$:} Recalling that the density of $X \sim \mathcal{N}(0,1)$ is $p_X(x) = (2\pi)^{-1/2} \exp(-\frac12 x^2)$, for $Z = XY$ where $X,Y \sim \mathcal{N}(0,1)$ are independent, 
\begin{align*}
p_Z(z) &= \frac{1}{2\pi}\int_{\mathbb{R}}\left|x\right|^{-1}e^{-\frac{1}{2}x^{2}}e^{-\frac{1}{2}z^{2}/x^{2}}\dd x\\
&=\frac{1}{\pi}\int_{0}^{\infty}x^{-1}e^{-\frac{1}{2}x^{2}-\frac{1}{2}z^{2}/x^{2}}\dd x\\
&=\frac{1}{\pi}\int_{0}^{\infty}x^{-1}e^{-\frac{1}{2}x^{2}-\frac{1}{2}z^{2}/x^{2}}\dd x.
\end{align*}
Let $u=\frac{1}{2}x^{2}$ so that $\dd u=x \dd x$ and
\[
K_{\nu}(z)=z^{\nu}\int_{0}^{\infty}x^{-2\nu-1}\exp\left(-\frac{1}{2}x^{2}-\frac{z^{2}}{2x^{2}}\right)\dd x.
\]
In particular, for any $z \in \mathbb{R}$,
\begin{equation}
\label{eq:BesselGauss}
K_{0}(|z|)=\int_{0}^{\infty}x^{-1}\exp\left(-\frac{1}{2}x^{2}-\frac{z^{2}}{2x^{2}}\right)\dd x,
\end{equation}
and so
\[
p_Z(z) = \frac{1}{\pi} K_0(|z|) \sim \frac{1}{\sqrt{2\pi|z|}} e^{-|z|}.
\]
\item \textbf{Density of $Z$:} Finally, by a change of variables, we note that the density of $X^{-1}$ where $X \sim \mathcal{N}(0,1)$ is $p_{X^{-1}}(x) = (2\pi)^{-1/2} x^{-2}\exp(-\frac{1}{2x^2})$. Therefore, the density of $Z = 1/(XY)$ where $X,Y \sim \mathcal{N}(0,1)$ are independent is given by
\begin{align*}
p_{Z}(z)&=\int_{\mathbb{R}}\frac{1}{\sqrt{2\pi}x^{2}}e^{-\frac{1}{2x^{2}}}\frac{x^{2}}{\sqrt{2\pi}z^{2}}e^{-\frac{x^{2}}{2z^{2}}}\frac{1}{\left|x\right|}\dd x\\
&=\frac{1}{2\pi z^{2}}\int_{\mathbb{R}}e^{-\frac{1}{2x^{2}}-\frac{x^{2}}{2z^{2}}}\frac{1}{\left|x\right|}\dd x\\
&=\frac{1}{\pi z^{2}}\int_{0}^{\infty}e^{-\frac{1}{2x^{2}}-\frac{x^{2}}{2z^{2}}}\frac{1}{x}\dd x\\
&=\frac{1}{\pi z^{2}}K_{0}(|z|^{-1}) \sim \sqrt{\frac{1}{2\pi}}|z|^{-3/2}e^{-|z|^{-1}},
\end{align*}
where we have once again used (\ref{eq:BesselGauss}). 
\end{itemize}
\end{document}